\documentclass[11pt]{article}
\usepackage{mathrsfs}
\usepackage{amsmath}
\usepackage{amssymb}
\usepackage{amsthm}
\usepackage{epsfig}
\usepackage{graphics}
\usepackage{graphicx}
\usepackage{txfonts}
\usepackage{amsfonts}
\usepackage{multirow,tabularx}
\usepackage{caption}
\usepackage{float}
\usepackage{comment}
\usepackage{latexsym}
\usepackage{indentfirst}
\usepackage{url}

\usepackage{color}

\begin{document}

\title{Solving Support Vector Machines in Reproducing Kernel Banach Spaces with Positive Definite Functions}

\author{Gregory E. Fasshauer, Fred J. Hickernell and Qi Ye\thanks{Corresponding author}}

\date{}

\maketitle


\newtheorem{theorem}{Theorem}[section]
\newtheorem{lemma}[theorem]{Lemma}
\newtheorem{corollary}[theorem]{Corollary}
\newtheorem{proposition}[theorem]{Proposition}

\theoremstyle{definition}
\newtheorem{definition}{Definition}[section]
\newtheorem{example}{Example}[section]

\theoremstyle{remark}
\newtheorem{remark}{Remark}[section]

\numberwithin{equation}{section}
\numberwithin{figure}{section}
\numberwithin{table}{section}


\def\RR{\mathbb{R}}       
\def\Rd{\mathbb{R}^d}     
\def\vx{\boldsymbol{x}}   
\def\vy{\boldsymbol{y}}   
\def\vc{\boldsymbol{c}}
\def\NN{\mathbb{N}}       
\def\CC{\mathbb{C}}       
\def\PP{\mathbb{P}}

\def\ud{\mathrm{d}}       

\def\vu{\boldsymbol{u}}
\def\vv{\boldsymbol{v}}

\def\vf{\boldsymbol{f}}
\def\vg{\boldsymbol{g}}

\def\v0{\boldsymbol{0}}

\def\vl{\boldsymbol{l}}

\newcommand{\norm}[1]{\left\lVert#1\right\rVert}  
\newcommand{\abs}[1]{\left\lvert#1\right\rvert}   


\def\Banach{\mathcal{B}}
\def\Hilbert{\mathcal{H}}

\def\Cont{\mathrm{C}}

\def\Leb{\mathrm{L}}

\def\vA{\mathsf{A}}
\def\vI{\mathsf{I}}
\def\vJ{\mathsf{J}}
\def\vV{\mathsf{V}}
\def\vD{\mathsf{D}}

\def\Span{\mathrm{span}}
\def\Sign{\mathrm{Sign}}
\def\Space{\mathcal{N}}
\def\Fun{\mathcal{F}}

\def\Real{\mathcal{R}e}

\def\Fourier{\mathcal{F}}

\def\vb{\boldsymbol{b}}
\def\vphi{\boldsymbol{\phi}}
\def\vk{\boldsymbol{k}}

\def\SI{\mathcal{SI}}

\def\Domain{\Omega}

\def\Schwartz{\mathscr{S}}

\def\Borel{\mathscr{B}}

\def\Fourier{\mathcal{F}}

\def\mT{\mathcal{T}}

\def\Ker{\mathcal{K}}



\begin{abstract}

In this paper we solve support vector machines in reproducing kernel Banach spaces with reproducing kernels defined on nonsymmetric domains instead of the traditional methods in reproducing kernel Hilbert spaces. Using the orthogonality of semi-inner-products, we can obtain the explicit representations of the dual (normalized-duality-mapping) elements of support vector machine solutions. In addition, we can introduce the reproduction property in a generalized native space by Fourier transform techniques such that it becomes a reproducing kernel Banach space, which can be even embedded into Sobolev spaces,  and its reproducing kernel is set up by the related positive definite function.
The representations of the optimal solutions of support vector machines (regularized empirical risks) in these reproducing kernel Banach spaces are formulated explicitly in terms of positive definite functions, and their finite numbers of coefficients can be computed by fixed point iteration.
We also give some typical examples of reproducing kernel Banach spaces induced by Mat\'ern functions (Sobolev splines) so that their support vector machine solutions are well computable as the classical algorithms. Moreover, each of their reproducing bases includes information from multiple training data points.
The concept of reproducing kernel Banach spaces offers us a new numerical tool for solving support vector machines.

\end{abstract}


\section{Introduction}

The theory and practice of kernel-based methods is a fast growing research area. They have been used for both scattered data approximation and machine learning. Applications come from such different fields as physics, biology, geology, meteorology and finance. The books~\cite{Buhmann2003,Fasshauer2007,Wahba1990,Wendland2005} show how to use (conditionally) positive definite kernels to construct interpolants for observation data sampled from some unknown functions in the native spaces induced by the kernel functions. In the books~\cite{Alpaydin2010,SteinwartChristmann2008}, the optimal support vector machine solutions are obtained in reproducing kernel Hilbert spaces (RKHSs), and these solutions are formulated in terms of the related reproducing kernels and given data values. Actually, as long as the same inner product is used, the concepts of native spaces and RKHSs are interchangeable. It is just that researchers in numerical analysis and statistical learning use different terminology and techniques to introduce those spaces. Moreover, the recent contributions~\cite{FasshauerYe2011Dist,FasshauerYe2011DiffBound,YePhD2012} develop a clear and detailed
framework for generalized Sobolev spaces and RKHSs by establishing a connection between Green functions and reproducing kernels.

Related to the current research work, \cite{Erickson2007,EricksonFasshauer2007,ZhangXuZhang2009} all generalize classical native spaces (RKHSs) to Banach spaces in different ways.
However, the reproducing property in generalized native spaces is not discussed in \cite{Erickson2007,EricksonFasshauer2007}, and \cite{ZhangXuZhang2009} does not mention how to use reproducing kernels to introduce the explicit forms of their reproducing kernel Banach spaces (RKBSs) analogous to the typical cases of RKHSs induced by Gaussian kernels and Sobolev-spline kernels, etc. Using \cite{ZhangXuZhang2009} it is therefore difficult to obtain explicit and simple support vector machine (SVM) solutions and perform practical computations. Following the results of these earlier authors, \cite[Section~6]{YePhD2012} tries to combine both of these ideas, and uses Fourier transform techniques to construct RKBSs.

In this paper we want to complete and extend the theoretical results in \cite[Section~6]{YePhD2012}. In addition, the RKBS given in Definition~\ref{d:RKBS} is different from that of~\cite{ZhangXuZhang2009}. Our RKBS can be \emph{one-sided} or \emph{two-sided} and its reproducing kernel $K$ can be defined on nonsymmetric domains, i.e., $K:\Domain_2\times\Domain_1\to\CC$, where $\Domain_1$ and $\Domain_2$ can be various subsets of $\RR^{d_1}$ and $\RR^{d_2}$, respectively (see Definition~\ref{d:RKBS}). Our RKBS is an extension of the RKHS and it does not require the reflexivity condition. The RKBS defined in~\cite{ZhangXuZhang2009} can be seen as a special case of the RKBS defined in this paper. According to Lemma~\ref{l:RKBS-opt-rep}, we can still obtain the optimal solution in the one-sided RKBS using the techniques of semi-inner-products.

It is well known that for given training data $D:=\left\{(\vx_j,y_j)\right\}_{j=1}^N$ the classical SVM (regularized empirical risk) in the RKHS $\Hilbert$ has the form
\[
\min_{f\in\Hilbert}\sum_{j=1}^NL\left(\vx_j,y_j,f(\vx_j)\right)+R\left(\norm{f}_{\Hilbert}\right),
\]
where $L$ is a \emph{loss function} and $R$ is a \emph{regularization function} (see Theorem~\ref{t:RKHS-opt-rep}).
In the same way we are able to apply an optimal recovery of RKBSs to solve SVMs in RKBSs.
Theorem~\ref{t:RKBS-opt-rep} establishes that the SVM in the right-sided RKBS $\Banach$ with the reproducing kernel $K:\Domain_2\times\Domain_1\to\CC$ based on the training data $D\subseteq\Domain_1\times\CC$ satisfies
\[
\min_{f\in\Banach}\sum_{j=1}^NL\left(\vx_j,y_j,f(\vx_j)\right)+R\left(\norm{f}_{\Banach}\right).
\]
Moreover, this problem has a unique optimal solution $s_{D,L,R}$ and its dual (normalized-duality-mapping) element $s_{D,L,R}^{\ast}$ is a linear combination of the reproducing kernel centered at the training data points $\left\{\vx_1,\ldots,\vx_N\right\}\subseteq\Domain_1$, i.e.,
\[
s_{D,L,R}^{\ast}(\vx)=\sum_{k=1}^Nc_kK(\vx,\vx_k),\quad \vx\in\Domain_2.
\]
According to Corollary~\ref{c:RKBS-opt-coef-fixed-point}, the coefficient vector $\vc:=\left(c_1,\cdots,c_N\right)^T$ of $s_{D,L,R}^{\ast}$ is a fixed point of the function $F_{D,L,R}^{\ast}:\RR^N\to\RR^N$ dependent of the differential loss function $L$ and the differential regularization function $R$, i.e., $F_{D,L,R}^{\ast}(\vc)=\vc$.
From this it is obvious that the SVM in the RKBS is the generalization of the classical method in the RKHS.

In Section~\ref{s:RKBS-PDF}, we show how to use a positive definite function $\Phi$ to set up different RKBSs $\Banach_{\Phi}^p(\Rd)$ and $\Banach_{\Phi}^p(\Domain)$ with $p>1$ whose two-sided reproducing kernel is given by $K(\vx,\vy)=\Phi(\vx-\vy)$ (see Theorems~\ref{t:RKBS-PDF} and~\ref{t:RKBS-PDF-Omega}). We can observe that $\Banach_{\Phi}^p(\Rd)$ is a kind of generalized native space. Furthermore, $\Banach_{\Phi}^p(\Rd)$ and $\Banach_{\Phi}^p(\Domain)$ coincide with the definition of RKBSs given in~\cite{ZhangXuZhang2009}.
The SVM solution $s_{D,L,R}$ in $\Banach_{\Phi}^p(\Rd)$ can be represented by the positive definite function $\Phi$, which means that we can obtain an explicit formula for the SVM solution $s_{D,L,R}$ in $\Banach_{\Phi}^p(\Rd)$ (see Theorem~\ref{t:RKBS-PDF-opt-rep}).
Corollary~\ref{c:RKBS-PDF-opt-coef-fixed-point} shows that the finite dimensional coefficients of the SVM solution $s_{D,L,R}$ can even be obtained by solving a fixed point iteration problem for differentiable loss functions and regularization functions.
Theorem~\ref{t:RKBS-PDF-Omega} and Corollary~\ref{c:RKBS-PDF-Omega} give some examples of reproducing kernels defined on nonsymmetric domains.
Corollary~\ref{c:RKBS-PDF-Sobolev} and~\ref{c:RKBS-PDF-Sobolev-Omega} provide that RKBSs can be embedded into Sobolev spaces for some special reproducing kernels, e.g., Sobolev-spline kernels (Mat\'ern functions).

The Mat\'ern functions represent a fast growing research area which has frequent applications in approximation theory and statistical learning,
and moreover, they are positive definite functions and (full-space) Green functions (see \cite{Fasshauer2007,FasshauerYe2011Dist,Stein1999,YePhD2012}).
In Section~\ref{s:Matern}, we solve the SVMs in the RKBSs of Mat\'ern functions. If $G_{\theta,n}$ is the Mat\'ern function with parameter $\theta>0$ and degree $n>3d/2$ then, according to our theoretical results, $\Banach_{G_{\theta,n}}^2(\Rd)$ is an RKHS, while $\Banach_{G_{\theta,n}}^4(\Rd)$ is only an RKBS. Their reproducing kernels, however, are the same Sobolev-spline kernel $K_{\theta,n}(\vx,\vy):=G_{\theta,n}(\vx-\vy)$. It is well known that the SVM solution in $\Banach_{G_{\theta,n}}^2(\Rd)\equiv\Hilbert_{G_{\theta,n}}(\Rd)$ has the explicit expression
\[
s_{D,L,R}(\vx):=\sum_{k=1}^Nc_kK_{\theta,n}(\vx,\vx_k),\quad \vx\in\Rd,
\]
(see Theorem~\ref{t:RKHS-opt-rep}). In this paper we discover a new fact that the SVM solution in $\Banach_{G_{\theta,n}}^4(\Rd)$ also has an explicit form, namely
\[
s_{D,L,R}(\vx)
=\sum_{k_1,k_2,k_3=1}^{N,N,N}c_{k_1}\overline{c_{k_2}}c_{k_3}\Ker_{\theta,3n}\left(\vx,\vx_{k_1},\vx_{k_2},\vx_{k_3}\right),
\quad \vx\in\Rd,
\]
where $\Ker_{\theta,3n}(\vx,\vy_1,\vy_2,\vy_3):=G_{\theta,3n}(\vx-\vy_1+\vy_2-\vy_3)$.
Section~\ref{s:Matern} shows that several other explicit representations of SVM solutions in the RKBS $\Banach_{G_{\theta,n}}^{p}(\Rd)$ are easily computable when $p$ is an even number. This discovery could lead to a new numerical tool for SVMs.

For the binary classification problems, it is well-known that the classical hinge loss is designed to maximize the $2$-norm margins by using the linear functions. However, we can not employ the hinge loss to set up the SVMs in order to maximize other $p$-norm margins. We guess that for applications to the problems that arise in current practice it will be necessary to construct loss functions depending on different kinds of RKBSs.

\begin{remark}\label{r:correction-mistake}
In this paper, the third author hopes to correct a mistake concerning the optimal recovery of RKBS $\Banach_{\Phi}^p(\Rd)$ mentioned in \cite[Section~6.2]{YePhD2012}. Theorem~\ref{t:RKBS-PDF-opt-rep} is the correction of \cite[Theorem~6.5]{YePhD2012}, which was the result of a misconception that the normalized duality mapping is linear. The main ideas and techniques used in the corrected version below are still the same as in~\cite{YePhD2012}. An updated version of~\cite{YePhD2012} has been posted on Ye's webpage.
\end{remark}

\section{Banach Spaces}\label{s:Banach}

In this section, we review some classical theoretical results for Banach spaces from \cite{Giles1967,James1947,Lumer1961,Megginson1998}. We denote the \emph{dual space} (the collection of all bounded linear functionals) of a Banach space $\Banach$ by $\Banach'$ and its \emph{dual bilinear product} as $\langle \cdot,\cdot \rangle_{\Banach}$, i.e.,
\[
\langle f,T \rangle_{\Banach}:=T(f), \quad \text{for all }T\in\Banach'\text{ and all }f\in\Banach.
\]
\cite[Theorem~1.10.7]{Megginson1998} states that $\Banach'$ is also a Banach space.

If the Banach spaces $\Banach_1$ and $\Banach_2$ are \emph{isometrically isomorphic} (equivalent), i.e., $\Banach_1\equiv\Banach_2$, then we can think of both spaces as being identical in the sense that their norms and their elements can be seen to be the same in both spaces (see \cite[Definition~1.4.13]{Megginson1998}). We say that $\Banach_1$ is \emph{embedded into} $\Banach_2$ if there exists a positive constant $C$ such that $\norm{f}_{\Banach_2}\leq C\norm{f}_{\Banach_1}$ for all $f\in\Banach_1\subseteq\Banach_2$ (see \cite[Section~1.25]{AdamsFournier2003}).

If the Banach space $\Banach$ is \emph{reflexive} (see \cite[Definition~1.11.6]{Megginson1998}), then we have $\Banach''\equiv\Banach$ and $\langle f,g \rangle_{\Banach}=\langle g,f \rangle_{\Banach'}$ for all $f\in\Banach$ and all $g\in\Banach'$. For example, the function space $\Leb_p(\Domain;\mu)$ defined on the positive measure space $(\Domain,\Borel_{\Domain},\mu)$ is a reflexive Banach space and its dual space is isometrically equivalent to $\Leb_q(\Domain;\mu)$ where $p,q>1$ and $p^{-1}+q^{-1}=1$ (see \cite[Example~1.10.2 and Theorem~1.11.10]{Megginson1998}). For the complex situation, the isometric isomorphism from $\Leb_p(\Domain;\mu)'$ onto $\Leb_q(\Domain;\mu)$ is antilinear.

We say that $\Banach$ is \emph{uniformly convex} if, for every $\epsilon>0$, there is $\delta>0$ such that
\[
\norm{\frac{f+g}{2}}_{\Banach}\leq 1-\delta,\text{ whenever }\norm{f}_{\Banach}=\norm{g}_{\Banach}=1\text{ and }\norm{f-g}_{\Banach}\geq\epsilon
\]
(see \cite[Definition~5.2.1]{Megginson1998}). According to \cite[Definition~5.4.1,~5.4.15 and Corollary~5.4.18]{Megginson1998},
$\Banach$ is said to be \emph{smooth} or \emph{G\^{a}teaux differentiable} if
\[
\lim_{\lambda\to0}{\frac{\norm{f+\lambda g}_{\Banach}-\norm{f}_{\Banach}}{\lambda}}\text{ exists },\quad\text{for all }f,g\in\Banach.
\]
A typical case is that $\Leb_p(\Domain;\mu)$ is uniformly convex and smooth if $1<p<\infty$.

It is well known that we can discuss the orthogonality in Banach spaces with a more general axiom system than that in Hilbert spaces. The papers~\cite{Giles1967,James1947,Lumer1961} show that every Banach space can be represented as a semi-inner-product space in order that the theories of Banach space can be penetrated by Hilbert space type arguments. A \emph{semi-inner-product} $[\cdot,\cdot]_{\Banach}:\Banach\times\Banach\to\CC$ defined on a Banach space $\Banach$ is given by
\begin{align*}
&(i)~[f+g,h]_{\Banach}=[f,h]_{\Banach}+[g,h]_{\Banach},\quad
(ii)~[f,f]_{\Banach}=\norm{f}_{\Banach}^2,\\
&(iii)~[\lambda f,g]_{\Banach}=\lambda[f,g]_{\Banach},~ [f,\lambda g]_{\Banach}=\overline{\lambda}[f,g]_{\Banach},\quad
(iv)~\abs{[f,g]_{\Banach}}\leq [f,f]_{\Banach}[g,g]_{\Banach},
\end{align*}
for all $f,g,h\in\Banach$ and all $\lambda\in\CC$. However, Hermitian symmetry of the semi-inner-product may not hold, i.e., $[f,g]_{\Banach}\neq\overline{[g,f]}_{\Banach}$. This indicates that the generality of the semi-inner-product in Banach space is a serious limitation for any extensive development that parallels the inner product of Hilbert space.

For example, a semi-inner-product of $\Leb_p(\Domain;\mu)$ with $1<p<\infty$ is given by
\[
[g,f]_{\Leb_p(\Domain;\mu)}=\frac{1}{\norm{f}_{\Leb_p(\Domain;\mu)}^{p-2}}\int_{\Domain}g(\vx)\overline{f(\vx)}\abs{f(\vx)}^{p-2}\ud\mu(\vx),
\quad \text{for all }f,g\in\Leb_p(\Domain;\mu),
\]
(see examples in \cite{Giles1967,James1947}).

We say that $f$ is \emph{orthogonal} to $g$ in a Banach space $\Banach$ if
\[
\norm{f+\lambda g}_{\Banach}\geq\norm{f}_{\Banach},\quad \text{for all }\lambda\in\CC,
\]
(see the definitions in \cite{Giles1967,James1947}). Suppose that the Banach space $\Banach$ is smooth. Using \cite[Theorem~2]{Giles1967}, we can determine that $f$ is orthogonal to $g$ if and only if $f$ is \emph{normal to} $g$, i.e.,
\[
[g,f]_{\Banach}=0.
\]

We can also obtain a representation theorem in Banach space by an adaptation of the representation theorem in Hilbert space.
Suppose that the Banach space $\Banach$ is uniformly convex and smooth.
According to \cite[Theorem~3 and 6]{Giles1967}, for every bounded linear functional $T\in\Banach'$, there exists a unique $f\in\Banach$ such that
\[
T(g)=\langle g,T \rangle_{\Banach}=[g,f]_{\Banach},\quad \text{for all }g\in\Banach,
\]
and $\norm{T}_{\Banach'}=\norm{f}_{\Banach}$. This mapping is also surjective. We call $T$ the \emph{normalized-duality-mapping element} of $f$ and rewrite it as $f^{\ast}:=T$. For convenience we simplify normalized-duality-mapping element to dual element in this paper.
The normalized duality mapping
is a one-to-one and norm-preserving mapping from $\Banach$ onto $\Banach'$. Note that this mapping is usually nonlinear. According to \cite[Theorem~7]{Giles1967}, the semi-inner-product of $\Banach'$ has the form $[f^{\ast},g^{\ast}]_{\Banach'}=[g,f]_{\Banach}$ for all $f^{\ast},g^{\ast}\in\Banach'$.
For example, the dual element of $f\in\Leb_p(\Domain;\mu)$ with $1<p<\infty$ is given by
\[
f^{\ast}=\frac{f(\vx)\abs{f(\vx)}^{p-2}}{\norm{f}_{\Leb_p(\Domain;\mu)}^{p-2}}\in\Leb_q(\Domain;\mu),
\]
where $q$ is the conjugate exponent of $p$.
Let $\Space$ be a subset of $\Banach$. We can check that $f$ is orthogonal to $\Space$ if and only if its dual element $f^{\ast}\in\Space^{\perp}=\left\{\eta\in\Banach':~\langle h,\eta \rangle_{\Banach}=0,\text{ for all }h\in\Space\right\}$, i.e.,
\[
[h,f]_{\Banach}=\langle h,f^{\ast} \rangle_{\Banach}=0,\quad \text{for all }h\in\Space.
\]

\section{Reproducing Kernels and Reproducing Kernel Hilbert Spaces}\label{s:RK-RKHS}

Most of the material presented in this section can be found in the monographs~\cite{Fasshauer2007,SteinwartChristmann2008,Wendland2005}. For the reader's convenience we repeat here what is essential to our discussion later on.

\begin{definition}[{\cite[Definition~10.1]{Wendland2005}}]\label{d:RKHS}
Let $\Domain\subseteq\Rd$ and $\Hilbert$ be a Hilbert space consisting of functions
$f:\Domain\to\CC$. $\Hilbert$ is called a \emph{reproducing kernel Hilbert space} (RKHS) and a kernel function $K:\Domain\times\Domain\to\CC$ is
called a \emph{reproducing kernel} for $\Hilbert$ if
\[
(i) \ K(\cdot,\vy)\in\Hilbert \text{ and } (ii) \ f(\vy)=(f,K(\cdot,\vy))_{\Hilbert},\quad\text{for
all }f\in\Hilbert\text{ and all }\vy\in\Domain,
\]
where $(\cdot,\cdot)_{\Hilbert}$ is used to denote the inner product of $\Hilbert$.
\end{definition}

\begin{remark}
In order to simplify our discussion and proofs, we let all kernel functions be complex-valued and all function spaces be composed of complex-valued functions in this paper. According to \cite[Proposition~1.9.3]{Megginson1998}, it is not difficult for us to restrict the theoretical results to real kernel functions and function spaces.
\end{remark}

\subsection{Optimal Recovery in Reproducing Kernel Hilbert Spaces}\label{s:opt-RKHS}

\begin{theorem}[{Representer theorem~\cite[Theorem~5.5]{SteinwartChristmann2008}}]\label{t:RKHS-opt-rep}
Let $\Hilbert$ be a reproducing kernel Hilbert space with a reproducing kernel $K$ defined on $\Domain\subseteq\Rd$, and a regularization function $R:[0,\infty)\to[0,\infty)$ be convex and strictly increasing. We choose the loss function $L:\Domain\times\CC\times\CC\to[0,\infty)$ such that $L(\vx,y,\cdot)$ is a convex map for any fixed $\vx\in\Domain$ and any fixed $y\in\CC$.
Given the data $D:=\left\{\left(\vx_1,y_1\right),\ldots,\left(\vx_N,y_N\right)\right\}$, with pairwise distinct data points $X=\left\{\vx_1,\ldots,\vx_N\right\}\subseteq\Domain$ and associated data values $Y=\left\{y_1,\ldots,y_N\right\}\subset\CC$, the optimal solution (support vector machine solution) $s_{D,L,R}$ of
\[
\min_{f\in\Hilbert}\sum_{j=1}^NL\left(\vx_j,y_j,f(\vx_j)\right)+R\left(\norm{f}_{\Hilbert}\right),
\]
has the explicit representation
\[
s_{D,L,R}(\vx)=\sum_{k=1}^Nc_kK(\vx,\vx_k),\quad \vx\in\Domain,
\]
for some coefficients $c_1,\ldots,c_N\in\CC$.
\end{theorem}

\subsection{Constructing Reproducing Kernel Hilbert Spaces by Positive Definite Functions}\label{s:RKHS-PDF}

\begin{definition}[{\cite[Definition~6.1]{Wendland2005}}]\label{d:PDF}
A continuous even function $\Phi:\Rd\to\CC$ is called \emph{positive definite} if, for all $N\in\NN$ and
all sets of pairwise distinct centers $X=\left\{\vx_1,\ldots,\vx_N\right\}\subset\Rd$, the quadratic form
\[
\sum_{j,k=1}^{N,N}c_j\overline{c_k}\Phi(\vx_j-\vx_k)=\vc^{\ast}\vA_{\Phi,X}\vc>0,\quad \text{for all }\vc\in\CC^N\backslash\{\v0\}.
\]
Here the interpolation matrix $\vA_{\Phi,X}:=\left(\Phi(\vx_j-\vx_k)\right)_{j,k=1}^{N,N}\in\CC^{N\times N}$ and $\vc^{\ast}=\overline{\vc}^T$.
\end{definition}

We say $\Phi$ is even if $\Phi(\vx)=\overline{\Phi(-\vx)}$. This shows that $\Phi$ is a positive definite function if and only if $\vA_{\Phi,X}$ is a positive definite matrix for any pairwise distinct finite set $X$ of data points in $\Rd$. The application and history of positive definite functions can be seen in the review paper~\cite{Fasshauer2011}. \cite[Section~10.2]{Wendland2005} shows how to use positive definite functions to construct RKHSs.

\begin{theorem}[{\cite[Theorem~6.11]{Wendland2005}}]\label{t:fourier-PDF}
Suppose that $\Phi\in\Cont(\Rd)\cap\Leb_1(\Rd)$. Then $\Phi$ is positive definite if and only if $\Phi$ is bounded and its Fourier transform $\hat{\Phi}$ is nonnegative and nonvanishing (nonzero everywhere).
\end{theorem}

\begin{remark}
In this paper, the Fourier transform of $f\in\Leb_1(\Rd)$ is defined by
\[
\hat{f}(\vx):=(2\pi)^{-d/2}\int_{\Rd}f(\vy)e^{-i\vx^T\vy}\ud\vy,
\]
where $i$ is the imaginary unit, i.e., $i^2=-1$.
\end{remark}

\begin{theorem}[{\cite[Theorem~10.12]{Wendland2005}}]\label{t:RKHS-PDF}
Suppose that $\Phi\in\Cont(\Rd)\cap\Leb_1(\Rd)$ is a positive definite function.
Then the space
\[
\Hilbert_{\Phi}(\Rd):=\left\{f\in\Leb_2(\Rd)\cap\Cont(\Rd):\hat{f}\big/\hat{\Phi}^{1/2}\in\Leb_2(\Rd)\right\},
\]
equipped with the norm
\[
\norm{f}_{\Hilbert_{\Phi}(\Rd)}:=\left((2\pi)^{-d/2}\int_{\Rd}\frac{\abs{\hat{f}(\vx)}^2}{\hat{\Phi}(\vx)}\ud\vx\right)^{1/2}
\]
is a reproducing kernel Hilbert space (native space) with reproducing kernel given by
\[
K(\vx,\vy):=\Phi(\vx-\vy),\quad \vx,\vy\in\Rd,
\]
where $\hat{\Phi}$ and $\hat{f}$ are the Fourier transforms of $\Phi$ and $f$, respectively. The inner product in $\Hilbert_{\Phi}(\Rd)$ has the form
\[
(f,g)_{\Hilbert}=(2\pi)^{-d/2}\int_{\RR}\frac{\hat{f}(\vx)\overline{\hat{g}(\vx)}}{\hat{\Phi}(\vx)}\ud\vx,\quad f,g\in\Hilbert_{\Phi}(\Rd).
\]
\end{theorem}

Using Fourier transform techniques similar to those in Theorem~\ref{t:RKHS-PDF}, we can employ positive definite functions to set up RKBSs (see Section~\ref{s:RKBS-PDF}).

\section{Reproducing Kernels and Reproducing Kernel Banach Spaces}\label{s:RK-RKBS}

Now we give the definition of RKBSs as a natural generalization of RKHSs by viewing the inner product as a dual bilinear product.

\begin{definition}\label{d:RKBS}
Let $\Domain_1$ and $\Domain_2$ be two subsets of $\RR^{d_1}$ and $\RR^{d_2}$ respectively, and $\Banach$ be a Banach space composed of functions
$f:\Domain_1\to\CC$, whose dual space $\Banach'$ is isometrically equivalent to a function space $\Fun$ with $g:\Domain_2\to\CC$. Denote that $K:\Domain_2\times\Domain_1\to\CC$ is a kernel function.

We call $\Banach$ a \emph{reproducing kernel Banach space} (RKBS) and $K$ its \emph{right-sided reproducing kernel} if
\[
\begin{split}
(i) \ K(\cdot,\vy)\in\Fun\equiv\Banach'\text{ and }
(ii) \ f(\vy)=\langle f,K(\cdot,\vy) \rangle_{\Banach},\quad\text{for
all }f\in\Banach\text{ and all }\vy\in\Domain_1.
\end{split}
\]

If the Banach space $\Banach$ reproduces from the other side, i.e.,
\[
(iii) \ \overline{K(\vx,\cdot)}\in\Banach \text{ and } (iv) \ \overline{g(\vx)}=\langle \overline{K(\vx,\cdot)},g \rangle_{\Banach},\quad \text{for all }g\in\Fun\equiv\Banach'\text{ and all }\vx\in\Domain_2,
\]
then $\Banach$ is called a \emph{reproducing kernel Banach space} and $K$ its \emph{left-sided reproducing kernel}.

For two-sided reproduction as above we say that $\Banach$ is a \emph{reproducing kernel Banach space} with the \emph{two-sided reproducing kernel} $K$.
\end{definition}

\begin{remark}\label{r:def-RKBS}
We know that the Riesz representer map on complex Hilbert space $\Hilbert$ is antilinear, i.e.,
\[
T_{\lambda g}(f)=\langle f,\lambda g \rangle_{\Hilbert}=\overline{\lambda}(f,g)_{\Hilbert}
=\overline{\lambda}\langle f,g \rangle_{\Hilbert}=\overline{\lambda}T_g(f),
\]
for all $f,g\in\Hilbert$ and all $\lambda\in\CC$.
Here we also let the isometrical isomorphism from the dual space $\Banach'$ onto the related function space $\Fun$ be antilinear.
Thus, the format of two-sided RKBSs coincides with complex RKHSs, i.e.,
\[
\begin{split}
&\langle \overline{K(\vy,\cdot)},f \rangle_{\Hilbert}=( \overline{K(\vy,\cdot)}, f )_{\Hilbert}
=\overline{(f,K(\cdot,\vy))}_{\Hilbert}=\overline{f(\vy)}, \quad \text{for all }f\in\Hilbert\text{ and all }\vy\in\Domain,\\
\end{split}
\]
which indicates that the RKHS is a special case of a two-sided RKBS.
\end{remark}

Why do we define our RKBSs differently from \cite[Definition~1]{ZhangXuZhang2009}?
The reason is that we can show the optimal recovery in an RKBS even if it is only one-sided. We do not require a  reflexivity condition for the definition of our RKBS. Moreover, since the dual space of a Hilbert space is isometrically equivalent to itself, we can choose the equivalent function space $\Fun\equiv\Hilbert$ such that the domain of the reproducing kernel $K$ is symmetric, i.e., $\Domain_2=\Domain_1$. Actually, the Banach space $\Banach$ is usually not equal to any equivalent function space $\Fun$ of its dual $\Banach'$ even though we only require them to be isomorphic. We naturally do not need any symmetry conditions in the Banach space. Therefore the nonsymmetric domain is used to define the RKBS $\Banach$ and its reproducing kernel $K$, i.e., $\Domain_2\neq\Domain_1$. The domain of $K$ is related to both $\Banach$ and $\Fun\equiv\Banach'$. If we choose a different $\Fun$ which is isometrically equivalent to the dual $\Banach'$, then we can obtain a different reproducing kernel $K$ of the RKBS $\Banach$ dependent on its equivalent dual space $\Fun$.

The functional $K(\cdot,\vy)$ can be seen as a point evaluation function $\delta_{\vy}$ defined on $\Banach$. This implies that $\delta_{\vy}$ is a bounded linear functional on $\Banach$, i.e., $\delta_{\vy}\in\Banach'$.
If the Banach space $\Banach$ is further uniformly convex and smooth, then its semi-inner-product and its normalized duality mapping are well-defined,
which can be used to set up the equivalent conditions of right-sided RKBSs, i.e.,
\[
\begin{split}
\delta_{\vy}\in\Banach'\equiv\Fun\text{ which indicates that }
f(\vy)=\langle f,\delta_{\vy} \rangle_{\Banach}=[ f,\delta_{\vy}^{\ast} ]_{\Banach},
\end{split}
\]
for all $f\in\Banach$ and all $\vy\in\Domain_1$ (see the discussions of the semi-inner products in Section~\ref{s:Banach}).

If $\Banach$ is a reflexive two-sided RKBS, then the equivalent dual space $\Fun$ of $\Banach$ is also a reflexive two-sided RKBS.
All RKBSs and reproducing kernels set up in Section~\ref{s:RKBS-PDF} satisfy the two-sided definition but their domains can be symmetric or nonsymmetric.

If a sequence $\left\{f_n\right\}_{n=1}^{\infty}\subset\Banach$ and $f\in\Banach$ such that
$\norm{f-f_n}_{\Banach}\to0$ when $n\to\infty$, then
\[
\abs{f(\vy)-f_n(\vy)}=\abs{\langle f-f_n,K(\cdot,\vy) \rangle_{\Banach}}
\leq\norm{K(\cdot,\vy)}_{\Banach'}\norm{f-f_n}_{\Banach}\to 0,\quad \vy\in\Domain_1,
\]
when $n\to\infty$. This means that convergence in the right-sided RKBS $\Banach$ implies pointwise convergence.

Suppose that $\Banach$ is a reflexive right-sided RKBS. We show that $\left\{K(\cdot,\vy): \ \vy\in\Domain_1\right\}$ is a linear vector space basis of $\Fun$ and $\Span\left\{K(\cdot,\vy): \ \vy\in\Domain_1\right\}$ is dense in $\Fun$.
Let $\Space$ be a completion (closure) of $\Span\left\{K(\cdot,\vy): \ \vy\in\Domain_1\right\}\subseteq\Fun\equiv\Banach'$ with its dual norm. Now we prove that $\Space\equiv\Fun\equiv\Banach'$. Since \cite[Theorem~1.10.7]{Megginson1998} provides that $\Fun$ is also a Banach space, we have $\Space\subseteq\Fun$. Assume that $\Space\subsetneqq\Fun$. According to \cite[Corollary~1.9.7]{Megginson1998} (application of Hahn-Banach extension theorems) there is an element $f\in\Banach\equiv\Banach''\equiv\Fun'$ such that $\norm{f}_{\Banach}=1$ and $f(\vy)=\langle f,K(\cdot,\vy) \rangle_{\Banach}=0$ for all $\vy\in\Domain_1$. We find the contradiction between $\norm{f}_{\Banach}=1$ and $f=0$. Thus the first assumption is not true and then we can conclude that $\Space\equiv\Fun\equiv\Banach'$, which indicates that $\left\{K(\cdot,\vy): \ \vy\in\Domain_1\right\}$ is a linear vector space basis of $\Fun$ and $\left\{\delta_{\vy}: \ \vy\in\Domain_1\right\}$ is a linear vector space basis of $\Banach'$.

\begin{example}\label{ex:simple}
We give a simple example of a two-sided RKBS. Let $\Domain_2=\Domain_1:=\left\{1,\cdots,n\right\}$ and $\vA\in\CC^{n\times n}$ be a symmetric positive definite matrix.
It can be decomposed into $\vA=\vV\vD\vV^{\ast}$, where $\vD$ is a positive diagonal matrix and $\vV$ is an orthogonal matrix. We choose $p,q>1$ such that $p^{-1}+q^{-1}=1$. Define
$\Banach:=\left\{f:\Domain_1\to\CC\right\}$ equipped with the norm
\[
\norm{f}_{\Banach}:=\norm{\vD^{-1/q}\vV^{\ast}\vf}_q, \quad\text{where } \vf:=\left(f(1),\cdots,f(n)\right)^T.
\]
We can check that $\Banach$ is a Banach space and its dual space $\Banach'$ is isometrically equivalent to
$\Fun:=\left\{g:\Domain_2\to\CC\right\}$ equipped with the norm
\[
\norm{g}_{\Banach'}:=\norm{\vD^{-1/p}\vV^{\ast}\vg}_p, \quad\text{where } \vg:=\left(g(1),\cdots,g(n)\right)^T.
\]
Moreover, its dual bilinear form is given by
\[
\langle f,g \rangle_{\Banach}=\vg^{\ast}\vA^{-1}\vf,\quad\text{for all }f\in\Banach\text{ and all }g\in\Banach'.
\]
If the kernel function is defined by
\[
K(j,k):=\vA_{jk},\quad j\in\Domain_2,~k\in\Domain_1,
\]
then the reproduction can easily be verified, i.e.,
\[
\langle f,K(\cdot,k) \rangle_{\Banach}=f(k),\quad k\in\Domain_1,
\quad\text{and}\quad
\langle \overline{K(j,\cdot)},g \rangle_{\Banach}=\overline{g(j)},\quad j\in\Domain_2.
\]
Therefore $\Banach$ is indeed a two-sided RKBS.

(In the same way, we can also employ the singular value decomposition of a nonsymmetric and nonsingular square matrix $\vA$ to introduce the two-sided RKBS.)
\end{example}

\subsection{Optimal Recovery in Reproducing Kernel Banach Spaces}\label{s:opt-RKBS}

It is well-known that any Hilbert space is uniformly convex and smooth. It is natural for us to assume the right-sided RKBS is further uniformly convex and smooth to discuss optimal recovery in it. The definitions of uniform convexity and smoothness of Banach spaces are given in Section~\ref{s:Banach}.

Given the pairwise distinct data points $X=\left\{\vx_1,\ldots,\vx_N\right\}\subseteq\Domain_1$ and the associated data values $Y=\left\{y_1,\ldots,y_N\right\}\subset\CC$, we define a subset of the right-sided RKBS $\Banach$ by
\[
\Space_{\Banach}(X,Y):=\left\{f\in\Banach:~f(\vx_j)=y_j,\text{ for all }j=1,\ldots,N\right\}.
\]
If $\Space_{\Banach}(X,Y)$ is the null set, then there is no meaning for the SVMs. So we need to assume that $\Space_{\Banach}(X,Y)$ is always non-null for the given data sites. Actually we can show that $\Space_{\Banach}(X,Y)$ is non-null for any data values $Y$ if and only if $\delta_{\vx_1},\ldots,\delta_{\vx_N}$ are linearly independent on $\Banach$ because $\sum_{k=1}^Nc_k\delta_{\vx_k}=0$ if and only if $\sum_{k=1}^Nc_kf(\vx_k)=0$ for all $f\in\Banach$, and moreover, $\vc=\left(c_1,\cdots,c_N\right)^T=\v0$ if and only if $\vb^{\ast}\vc=0$ for all $\vb\in\CC^N$.

In this section, we suppose that $\delta_{\vx_1},\ldots,\delta_{\vx_N}$ are always \emph{linearly independent} on $\Banach$ for the given pairwise distinct data points $X$, which is equivalent to the fact that $K(\cdot,\vx_1),\ldots,K(\cdot,\vx_N)$ are linearly independent.
We use the techniques of~\cite[Theorem~19]{ZhangXuZhang2009} to verify the following lemma.

\begin{lemma}\label{l:RKBS-opt-rep}
Let $\Banach$ be a reproducing kernel Banach space with a right-sided reproducing kernel $K$ defined on $\Domain_2\times\Domain_1\subseteq\RR^{d_2}\times\RR^{d_1}$. Suppose that $\Banach$ is uniformly convex and smooth.
Given the data $D:=\left\{\left(\vx_1,y_1\right),\ldots,\left(\vx_N,y_N\right)\right\}$ with pairwise distinct data points $X=\left\{\vx_1,\ldots,\vx_N\right\}\subseteq\Domain_1$ and associated data values $Y=\left\{y_1,\ldots,y_N\right\}\subset\CC$, the dual element $s_D^{\ast}$ of the unique optimal solution
\begin{equation}\label{e:opt-simple}
s_D:=\underset{f\in\Banach}{\text{argmin}}\left\{\norm{f}_{\Banach}:~f(\vx_j)=y_j,\text{ for all }j=1,\ldots,N\right\},
\end{equation}
is the linear combination of $K(\cdot,\vx_1),\ldots,K(\cdot,\vx_N)$, i.e.,
\[
s_D^{\ast}(\vx)=\sum_{k=1}^Nc_kK(\vx,\vx_k),\quad \vx\in\Domain_2.
\]

\end{lemma}
%
\begin{proof}
We first prove the uniqueness of the optimal solution of the minimization problem~\eqref{e:opt-simple}. Let us assume that the minimization problem~\eqref{e:opt-simple} has two optimal solutions $s_1,s_2\in\Banach$ with $s_1\neq s_2$. Since $\Banach$ is uniformly convex, \cite[Corollary~5.1.12]{Megginson1998} provides that $\norm{\frac{1}{2}\left(s_1+s_2\right)}_{\Banach}<\frac{1}{2}\norm{s_1}_{\Banach}+\frac{1}{2}\norm{s_2}_{\Banach}$. Then $\norm{s_1}_{\Banach}=\norm{s_2}_{\Banach}$ shows for $s_3:=\frac{1}{2}\left(s_1+s_2\right)$ that
$\norm{s_3}_{\Banach}<\norm{s_1}_{\Banach}$ and $s_3\in\Space_{\Banach}(X,Y)$,
i.e., $s_1$ is not an optimal solution of the minimization problem~\eqref{e:opt-simple}. The assumption that there are two minimizers is false.

Next we show the existence of the minimizer. The minimization problem~\eqref{e:opt-simple} is equivalent to $\min_{f\in\Space_{\Banach}(X,Y)}\norm{f}_{\Banach}$. Since convergence in a one-sided RKBS $\Banach$ implies pointwise convergence, we can check that $\Space_{\Banach}(X,Y)$ is a closed convex subset of $\Banach$. Combining this with the uniform convexity of $\Banach$, \cite[Corollary~5.2.17]{Megginson1998} shows that $\Space_{\Banach}(X,Y)$ is a Chebyshev set (see \cite[Definition~5.1.17]{Megginson1998}).
Thus an optimal solution $\min_{f\in\Space_{\Banach}(X,Y)}\norm{f}_{\Banach}$ exists.

Because $\Space_{\Banach}(X,Y)+\Space_{\Banach}(X,\{0\})=\Space_{\Banach}(X,Y)$ and $\Space_{\Banach}(X,\{0\})$ is a closed subspace of $\Banach$ we can determine that the optimal solution $s_D$ is orthogonal to $\Space_{\Banach}(X,\{0\})$, i.e., $\norm{s_D+h}_{\Banach}\geq\norm{s_D}_{\Banach}$ for all $h\in\Space_{\Banach}(X,\{0\})$. Since $\Banach$ is uniformly convex and smooth, the dual element $s_D^{\ast}$ of $s_D$ is well-defined and
\[
[h,s_D]_{\Banach}=\langle h,s_D^{\ast}\rangle_{\Banach}=0,\text{ for all }h\in\Space_{\Banach}(X,\{0\}),
\]
which implies that
\[
s_D^{\ast}\in\Space_{\Banach}(X,\{0\})^{\perp}=\left\{g\in\Fun\equiv\Banach':~\langle h,g\rangle_{\Banach}=0,\text{ for all }h\in\Space_{\Banach}(X,\{0\})\right\}.
\]
It is obvious that
\begin{align*}
&\Space_{\Banach}(X,\{0\})
=\left\{f\in\Banach:~f(\vx_j)=\langle f,K(\cdot,\vx_j)\rangle_{\Banach}=0,~j=1,\ldots,N\right\}\\
=&\left\{f\in\Banach:~\langle f,h\rangle_{\Banach}=0,~\text{for all }h\in\Span\left\{K(\cdot,\vx_k)\right\}_{k=1}^N\right\}
={}^{\perp}\Span\left\{K(\cdot,\vx_k)\right\}_{k=1}^N.
\end{align*}
According to \cite[Proposition~1.10.15]{Megginson1998}, we have
\[
s_D^{\ast}\in\left({}^{\perp}\Span\left\{K(\cdot,\vx_k)\right\}_{k=1}^N\right)^{\perp}
=\Span\left\{K(\cdot,\vx_1),\ldots,K(\cdot,\vx_N)\right\}.
\]
Here $\Space_1^{\perp}$ and ${}^{\perp}\Space_2$ denote the annihilator of $\Space_1$ in $\Banach'$ and the annihilator of $\Space_2$ in $\Banach$, respectively, where $\Space_1\subseteq\Banach$ and $\Space_2\subseteq\Banach'$ (see~\cite[Definition~1.10.14]{Megginson1998}).
\end{proof}

Now we verify the representer theorem for SVMs in a right-sided RKBS.

\begin{theorem}\label{t:RKBS-opt-rep}
Let $\Banach$ be a reproducing kernel Banach space with a right-sided reproducing kernel $K$ defined on $\Domain_2\times\Domain_1\subseteq\RR^{d_2}\times\RR^{d_1}$, and a regularization function $R:[0,\infty)\to[0,\infty)$ be convex and strictly increasing. Suppose that $\Banach$ is uniformly convex and smooth. We choose the loss function $L:\Domain_1\times\CC\times\CC\to[0,\infty)$ such that $L(\vx,y,\cdot)$ is a convex map for any fixed $\vx\in\Domain_1$ and any fixed $y\in\CC$.
Given the data $D:=\left\{\left(\vx_1,y_1\right),\ldots,\left(\vx_N,y_N\right)\right\}$ with pairwise distinct data points $X=\left\{\vx_1,\ldots,\vx_N\right\}\subseteq\Domain_1$ and associated data values $Y=\left\{y_1,\ldots,y_N\right\}\subset\CC$, the dual element of the unique optimal solution (support vector machine solution) $s_{D,L,R}$ of
\begin{equation}\label{e:opt-general}
\min_{f\in\Banach}\sum_{j=1}^NL\left(\vx_j,y_j,f(\vx_j)\right)+R\left(\norm{f}_{\Banach}\right),
\end{equation}
has the explicit representation
\[
s_{D,L,R}^{\ast}(\vx)=\sum_{k=1}^Nc_kK(\vx,\vx_k),\quad \vx\in\Domain_2,
\]
for some coefficients $c_1,\ldots,c_N\in\CC$.
\end{theorem}
%
\begin{proof}

Let
\[
T_{D,L,R}(f):=\sum_{j=1}^NL\left(\vx_j,y_j,f(\vx_j)\right)+R\left(\norm{f}_{\Banach}\right),\quad f\in\Banach.
\]
The minimization problem~\eqref{e:opt-general} is equivalent to $\min_{f\in\Banach}T_{D,L,R}(f)$.
Since $\Banach$ is uniformly convex and $R$ is convex and strictly increasing, the regularization $f\mapsto R\left(\norm{f}_{\Banach}\right)$ is continuous and strictly convex. Because the $\Banach$-norm convergence implies the pointwise convergence and $L\left(\vx_j,y_j,\cdot\right)$ is convex for all $j=1,\ldots,N$, the mapping $f\mapsto \sum_{j=1}^NL\left(\vx_j,y_j,f(\vx_j)\right)$ is also continuous and convex.
This indicates the continuity and strict convexity of $T_{D,L,R}$.
Using the increasing property of $R$, we can check that the set $\left\{f\in\Banach:~T_{D,L,R}(f)\leq T_{D,L,R}(0)\right\}$ is nonempty and bounded.
Moreover, the uniformly convex norm implies its reflexivity by the Milman-Pettis Theorem~\cite[Theorem~5.2.15]{Megginson1998}.
Thus the existence of minimizers theorem~\cite[Theorem~A.6.9]{SteinwartChristmann2008} gives the existence of the unique solution $s_{D,L,R}$ to minimize $T_{D,L,R}$ over $\Banach$.

We fix any $f\in\Banach$ and let $D_f:=\left\{(\vx_k,f(\vx_k))\right\}_{k=1}^N$. According to Lemma~\ref{l:RKBS-opt-rep}, there exists an element $s_{D_f}$ whose dual element $s_{D_f}^{\ast}\in\Span\left\{K(\cdot,\vx_k)\right\}_{k=1}^N$ such that $s_{D_f}$ interpolates the data values $\left\{f(\vx_k)\right\}_{k=1}^N$ at the centers points $X=\left\{\vx_k\right\}_{k=1}^N$ and $\norm{s_{D_f}}_{\Banach}\leq\norm{f}_{\Banach}$. This indicates that
\[
T_{D,L,R}(s_{D_f})
\leq
T_{D,L,R}(f).
\]
Therefore the dual element $s_{D,L,R}^{\ast}$ of the optimal solution $s_{D,L,R}$
of the minimization problem~\eqref{e:opt-general}
belongs to $\Span\left\{K(\cdot,\vx_k)\right\}_{k=1}^N$.

\end{proof}

\begin{remark}\label{r:RKBS-opt-rep}
Since $K(\cdot,\vx_j)$ can be seen as a point evaluation functional $\delta_{\vx_j}$ defined on $\Banach$, it indicates that the dual element of $s_{D,L,R}$ can be also written as a linear combination of $\delta_{\vx_1},\ldots,\delta_{\vx_N}$, i.e.,
$
s_{D,L,R}^{\ast}=\sum_{j=1}^Nc_j\delta_{\vx_j}.
$

The uniform convexity and smoothness of $\Banach$ imply the uniform convexity and smoothness of its dual $\Banach'\equiv\Fun$.
If $\Banach$ is a left-sided RKBS satisfying uniform convexity and smoothness conditions, then we can further perform optimal recovery in $\Fun$ in the same way, i.e.,
the dual element of the optimal solution (SVM solution) of
\[
\min_{g\in\Fun\equiv\Banach'}\sum_{j=1}^N\tilde{L}\left(\vx_j,y_j,\overline{g(\vx_j)}\right)+R\left(\norm{g}_{\Banach'}\right),
\]
is a linear combination of $\overline{K(\vx_1,\cdot)},\ldots,\overline{K(\vx_N,\cdot)}$, where $X=\left\{\vx_1,\ldots,\vx_N\right\}\subseteq\Domain_2$ and $\tilde{L}:\Domain_2\times\CC\times\CC\to[0,\infty)$.

Moreover, since the normalized duality mapping is an identity mapping on the Hilbert space and the reproducing kernel of an RKHS is symmetric, optimal recovery in RKBSs as in Theorem~\ref{t:RKBS-opt-rep} can be seen as a generalization of optimal recovery in RKHSs as in Theorem~\ref{t:RKHS-opt-rep}.
\end{remark}

Since the normalized duality mapping is one-to-one, for any fixed $\vc\in\CC^N$, there exists an unique $s_{\vc}\in\Banach$ such that its dual element has the form $s_{\vc}^{\ast}=\sum_{k=1}^Nc_kK(\cdot,\vx_k)=\vk_X^T\vc$, where $\vk_X:=\left(K(\cdot,\vx_1),\cdots,K(\cdot,\vx_N)\right)^T$ and $\vc:=\left(c_1,\cdots,c_N\right)^T$. According to Theorem~\ref{t:RKBS-opt-rep}, the SVM~\eqref{e:opt-general} can be transformed to solve a finite-dimensional optimization problem, i.e.,
\[
\vc_{opt}:=\underset{\vc\in\CC^N}{\text{argmin}}\sum_{j=1}^NL\left(\vx_j,y_j,s_{\vc}(\vx_j)\right)+R\left(\norm{s_{\vc}}_{\Banach}\right),
\]
and the dual element of the SVM solution has the form $s_{D,L,R}^{\ast}=\vk_X^T\vc_{opt}$.

Now we want to show that these optimal coefficients $\vc_{opt}$ can be computed by a fixed point iteration method similar as in~\cite{MicchelliShenXu2011}.
Suppose that $L(\vx,y,\cdot)\in\Cont^1(\CC)$ for all $\vx\in\Domain_1$ and all $y\in\CC$, and $R\in\Cont^1([0,\infty))$. Let
\[
\phi_j^{\ast}(\vc):=[K(\cdot,\vx_j),\vk_X^T\vc]_{\Banach'}=[K(\cdot,\vx_j),s_{\vc}^{\ast}]_{\Banach'},\quad
\vc\in\CC^N,\quad j=1,\ldots,N,
\]
and
\[
L'(\vx,y,t):=\frac{\ud}{\ud t}L(\vx,y,t),\quad \vx\in\Domain_1,\quad y\in\CC,
\]
where $\frac{\ud}{\ud t}$ represents the Wirtinger derivative defined by
\[
\frac{\ud}{\ud t}:=\frac{1}{2}\left(\frac{\ud}{\ud u}-i\frac{\ud}{\ud v}\right),\quad
\text{where }t=u+iv\text{ with }i^2=-1\text{ and }u,v\in\RR.
\]
Thus we have
\[
s_{\vc}(\vx_j)=\langle s_{\vc}, K(\cdot,\vx_j) \rangle_{\Banach}
=[s_{\vc}, K(\cdot,\vx_j)^{\ast}]_{\Banach}=[K(\cdot,\vx_j),s_{\vc}^{\ast}]_{\Banach'}=\phi_j^{\ast}(\vc),
\quad j=1,\ldots,N,
\]
and
\[
\norm{s_{\vc}}_{\Banach}^2=[s_{\vc},s_{\vc}]_{\Banach}=\langle s_{\vc},s_{\vc}^{\ast} \rangle_{\Banach}
=\sum_{j=1}^N\overline{c_j}\langle s_{\vc},K(\cdot,\vx_j) \rangle_{\Banach}=\sum_{j=1}^N\overline{c_j}\phi_j^{\ast}(\vc)=\vc^{\ast}\vphi^{\ast}(\vc),
\]
where $\vphi^{\ast}:=\left(\phi_1^{\ast},\cdots,\phi_N^{\ast}\right)^T$. Denote that
\[
\mT_{D,L,R}^{\ast}(\vc):=\sum_{j=1}^NL\left(\vx_j,y_j,\phi_j^{\ast}(\vc)\right)+R\left(\sqrt{\vc^{\ast}\vphi^{\ast}(\vc)}\right)
=\sum_{j=1}^NL\left(\vx_j,y_j,s_{\vc}(\vx_j)\right)+R\left(\norm{s_{\vc}}_{\Banach}\right).
\]
Since $\vc_{opt}$ is the global minimizer of $\mT_{D,L,R}^{\ast}$ over $\CC^N$, $\vc_{opt}$ is a stationary point of $\mT_{D,L,R}^{\ast}$, i.e., $\nabla\mT_{D,L,R}^{\ast}(\vc_{opt})=0$. We compute the gradient of $\mT_{D,L,R}^{\ast}$ by Wirtinger partial derivatives, i.e.,
\[
\nabla\mT_{D,L,R}^{\ast}(\vc)^T
=
\vl'_{D}\left(\vphi^{\ast}(\vc)\right)^T\nabla\vphi^{\ast}(\vc)
+\frac{R'\left(\sqrt{\vc^{\ast}\vphi^{\ast}(\vc)}\right)}{4\sqrt{\vc^{\ast}\vphi^{\ast}(\vc)}}
\vc^{\ast}\nabla\vphi^{\ast}(\vc),
\]
where
$\vl'_D\left(\vphi^{\ast}\right):=\left(L'(\vx_1,y_1,\phi_1^{\ast}),\cdots,L'(\vx_N,y_N,\phi_N^{\ast})\right)^T$ and
$\nabla\vphi^{\ast}:=\left(\frac{\partial}{\partial c_k}\phi_j^{\ast}\right)_{j,k=1}^{N,N}$ is the Jacobian (gradient) matrix of $\vphi^{\ast}$ by Wirtinger partial derivatives.
The optimal solution $\vc_{opt}$ is also a fixed point of the function $F_{D,L,R}^{\ast}$, i.e.,
\[
F_{D,L,R}^{\ast}(\vc_{opt})=\vc_{opt},
\]
where
\begin{equation}\label{e:opt-coef-fixed-point}
F_{D,L,R}^{\ast}(\vc):=
\vc+\nabla\mT_{D,L,R}^{\ast}(\vc),
\quad \vc\in\CC^N\backslash\{\v0\}.
\end{equation}

\begin{corollary}\label{c:RKBS-opt-coef-fixed-point}
Suppose that the loss function $L(\vx,y,\cdot)\in\Cont^1(\CC)$ for all $\vx\in\Domain_1$ and all $y\in\CC$, and the regularization function $R\in\Cont^1([0,\infty))$.
Then the coefficients $\vc$ of the dual element $s_{D,L,R}^{\ast}$ of the support vector machine solution $s_{D,L,R}$ given in Theorem~\ref{t:RKBS-opt-rep} is a fixed point of the function $F_{D,L,R}^{\ast}$ defined in Equation~\eqref{e:opt-coef-fixed-point}, i.e., $F_{D,L,R}^{\ast}(\vc)=\vc$.
\end{corollary}
%
\begin{remark}\label{r:RKBS-opt-coef-fixed-point}

Even though we can obtain the coefficients of $s_{D,L,R}^{\ast}$ by the fixed point iteration method, it is still difficult for us to recover the explicit form $s_{D,L,R}$ in many cases. In Section~\ref{s:RKBS-PDF} we discuss how to obtain the SVM solutions in RKBSs induced by positive definite functions (see Theorem~\ref{t:RKBS-PDF-opt-rep}). In that setting the coefficients of the explicit form are also computable by a fixed point iteration method for differentiable loss functions and regularization functions.

\end{remark}

\section{Constructing Reproducing Kernel Banach Spaces by Positive Definite Functions}\label{s:RKBS-PDF}

Now we construct RKBSs based on positive definite functions in a way similar to the construction of RKHSs in Theorem~\ref{t:RKHS-PDF}.
Let $1<p,q<\infty$ and $p^{-1}+q^{-1}=1$. Suppose that $\Phi\in\Cont(\Rd)\cap\Leb_1(\Rd)$ is a positive definite function.
According to Theorem~\ref{t:fourier-PDF}, we know that $\hat{\Phi}\in\Leb_1(\Rd)\cap\Cont(\Rd)$ is nonnegative and nonvanishing.
We define
\begin{equation}\label{eq:BPhip-Def}
\begin{split}
\Banach_{\Phi}^p(\Rd):=&\left\{f\in\Cont(\Rd)\cap\SI: \ \text{the distributional Fourier transform $\hat{f}$ of $f$ is}\right.\\
&\left.\text{a measurable function defined on $\Rd$ such that }
\hat{f}\big/\hat{\Phi}^{1/q}\in\Leb_q(\Rd)\right\},
\end{split}
\end{equation}
equipped with the norm
\[
\norm{f}_{\Banach_{\Phi}^p(\Rd)}:=\left((2\pi)^{-d/2}\int_{\Rd}\frac{\abs{\hat{f}(\vx)}^q}{\hat{\Phi}(\vx)}\ud\vx\right)^{1/q},
\]
where $\SI$ is the collection of all slowly increasing functions (see \cite[Definition~5.19]{Wendland2005}). We define $\Banach_{\Phi}^q(\Rd)$ in an analogous way as above.

\begin{remark}
Following the theoretical results of
\cite[Section~7.1]{Hormander2003I} and \cite[Section~1.3]{SteinWeiss1971} we can define the
distributional Fourier transform $\hat{T}\in\Schwartz'$ of
the tempered distribution $T\in\Schwartz'$ by
\[
\langle \gamma, \hat{T} \rangle_{\Schwartz} :=\langle \hat{\gamma},T \rangle_{\Schwartz},
\quad\text{for all }\gamma\in\Schwartz,
\]
where $\Schwartz$ is the Schwartz space (see~\cite[Definition~5.17]{Wendland2005}) and $\Schwartz'$ is its dual space with the dual bilinear form $\langle\cdot,\cdot\rangle_{\Schwartz}$. We can also verify that $\Cont(\Rd)\cap\SI\subset\Leb_1^{loc}(\Rd)\cap\SI$ is embedded into $\Schwartz'$.
\end{remark}

When $p\geq q$, then $\hat{\Phi}\in\Leb_1(\Rd)\cap\Cont(\Rd)$ implies that $\hat{\Phi}^{p/q}\in\Leb_1(\Rd)$ which will be used in the proof of the following theorem.
We also need to impose an additional symmetry condition on $\hat{\Phi}^{q/p}\in\Leb_1(\Rd)$ which is needed in the proof.
Since $p/q=p-1$ and $q/p=q-1$, this condition can be represented as $\hat{\Phi}^{\min\{p,q\}-1}\in\Leb_1(\Rd)$.

Since we can denote the positive measure $\mu$ on $\Rd$ as
\[
\mu(A):=(2\pi)^{-d/2}\int_{A}\frac{\ud\vx}{\hat{\Phi}(\vx)},\quad\text{for any open set $A$ of $\Rd$}.
\]
\cite[Example~1.2.6]{Megginson1998} provides that the space $\Leb_q(\Rd;\mu)$ is well-defined on the positive measure space $(\Rd,\Borel_{\Rd},\mu)$, i.e.,
\[
\Leb_q(\Rd;\mu):=\left\{f:\Rd\to\CC: \ f\text{ is measurable and }\int_{\Rd}\abs{f(\vx)}^q\ud\mu(\vx)<\infty\right\},
\]
equipped with the norm
\[
\norm{f}_{\Leb_q(\Rd;\mu)}:=\left(\int_{\Rd}\abs{f(\vx)}^q\ud\mu(\vx)\right)^{1/q}.
\]
$\Leb_p(\Rd;\mu)$ is also defined in an analogous way. \cite[Example~1.10.2 and Theorem~1.10.7]{Megginson1998} show that $\Leb_q(\Rd;\mu)$ is a Banach space and its dual space $\Leb_q(\Rd;\mu)'$ is isometrically equivalent to $\Leb_p(\Rd;\mu)$. In analogy to the representation theorem on Hilbert space, the bounded linear functional $T_g\in\Leb_q(\Rd;\mu)'$ associated with $g\in\Leb_p(\Rd;\mu)$ is given by
\[
T_g(f):=\int_{\Rd}f(\vx)\overline{g(\vx)}\ud\mu(\vx),\quad\text{for all } f\in\Leb_q(\Rd;\mu).
\]
Here, this isometric isomorphism from $\Leb_q(\Rd;\mu)'$ onto $\Leb_p(\Rd;\mu)$ is antilinear, just as the dual of complex Hilbert spaces, i.e.,
\[
T_{\lambda g}(f)=
\int_{\Rd}f(\vx)\overline{\lambda g(\vx)}\ud\mu(\vx)
=\overline{\lambda}T_{g}(f),\quad \text{for all }f\in\Leb_q(\Rd;\mu)\text{ and all }\lambda\in\CC.
\]

If we can show that $\Banach_{\Phi}^p(\Rd)$ and $\Leb_q(\Rd;\mu)$ are isometrically isomorphic, then $\Banach_{\Phi}^p(\Rd)$ is a Banach space and its dual space $\Banach_{\Phi}^p(\Rd)'$ is isometrically equivalent to $\Leb_p(\Rd;\mu)$. One can argue analogously for $\Banach_{\Phi}^q(\Rd)\equiv\Leb_p(\Rd;\mu)$. If we can further verify the two-sided reproduction of $\Banach_{\Phi}^p(\Rd)$, then $\Banach_{\Phi}^p(\Rd)$ is a two-sided RKBS.

\begin{theorem}\label{t:RKBS-PDF}
Let $1<p,q<\infty$ and $p^{-1}+q^{-1}=1$. Suppose that $\Phi\in\Leb_1(\Rd)\cap\Cont(\Rd)$ is a positive definite function on $\Rd$ and that $\hat{\Phi}^{\min\{p,q\}-1}\in\Leb_1(\Rd)$.
Then $\Banach_{\Phi}^p(\Rd)$ given in Equation~\eqref{eq:BPhip-Def} is a reproducing kernel Banach space with the two-sided reproducing kernel
\[
K(\vx,\vy):=\Phi(\vx-\vy),\quad \vx,\vy\in\Rd.
\]
Its dual space $\Banach_{\Phi}^p(\Rd)'$ and $\Banach_{\Phi}^q(\Rd)$ are isometrically isomorphic.
Moreover, $\Banach_{\Phi}^p(\Rd)$ is uniformly convex and smooth.

In particular, when $p=2$ then $\Banach_{\Phi}^2(\Rd)=\Hilbert_{\Phi}(\Rd)$ is a reproducing kernel Hilbert space as in Theorem~\ref{t:RKHS-PDF}.
\end{theorem}
%
\begin{proof}
For convenience, we assume that $p\geq q$.
We first prove that $\Banach_{\Phi}^p(\Rd)$ and $\Leb_q(\Rd;\mu)$ are isometrically isomorphic. The Fourier transform map can be seen as a one-to-one map from $\Banach_{\Phi}^p(\Rd)$ into $\Leb_q(\Rd;\mu)$. We can check the equality of their norm
\[
\norm{f}_{\Banach_{\Phi}^p(\Rd)}=\left((2\pi)^{-d/2}\int_{\Rd}\frac{\abs{\hat{f}(\vx)}^q}{\hat{\Phi}(\vx)}\ud\vx\right)^{1/q}
=\left(\int_{\Rd}\abs{\hat{f}(\vx)}^q\ud\mu(\vx)\right)^{1/q}=
\norm{\hat{f}}_{\Leb_q(\Rd;\mu)}.
\]
So the Fourier transform map is an isometric isomorphism. Now we verify that the Fourier transform map is surjective.
Fix any $h\in\Leb_q(\Rd;\mu)$. We want to find an element in $\Banach_{\Phi}^p(\Rd)$ whose Fourier transform is equal to $h$.
We conclude that $h\in\Leb_1(\Rd)$ because
\[
\int_{\Rd}\abs{h(\vx)}\ud\vx\leq
\left(\int_{\Rd}\frac{\abs{h(\vx)}^q}{\hat{\Phi}(\vx)}\ud\vx\right)^{1/q}
\left(\int_{\Rd}\hat{\Phi}(\vx)^{p/q}\ud\vx\right)^{1/p}<\infty.
\]
Thus, the inverse Fourier transform of $h$ given as $\check{h}(\vx)=(2\pi)^{-d/2}\int_{\Rd}h(\vy)e^{i\vx^T\vy}\ud\vy$ is well-defined and an element of $\Cont(\Rd)\cap\SI$. This indicates that $\hat{\check{h}}=h$ and $\check{h}\in\Banach_{\Phi}^p(\Rd)$ because $\langle \hat{\check{h}},\gamma \rangle_{\Schwartz}=\langle h,\check{\hat{\gamma}} \rangle_{\Schwartz}=\langle h,\gamma \rangle_{\Schwartz}$ for all $\gamma\in\Schwartz$. Therefore $\Banach_{\Phi}^p(\Rd)$ is isometrically equivalent to $\Leb_q(\Rd;\mu)$.

Using $\hat{\Phi}^{q/p}=\hat{\Phi}^{q-1}\in\Leb_1(\Rd)$ we can also prove that $\Banach_{\Phi}^q(\Rd)\equiv\Leb_p(\Rd;\mu)$ in an analogous way. Therefore $\Banach_{\Phi}^q(\Rd)$ is isometrically equivalent to the dual space of $\Banach_{\Phi}^p(\Rd)$.

We fix any $\vy\in\Rd$. The Fourier transform of $K(\cdot,\vy)$ is equal to $\hat{k}_{\vy}(\vx):=\hat{\Phi}(\vx)e^{-i\vx^T\vy}$. Since $\hat{\Phi}^{p-1}\in\Leb_1(\Rd)$ we have $\hat{k}_{\vy}\in\Leb_p(\Rd;\mu)$.
Thus $K(\cdot,\vy)$ can be seen as an element of $\Banach_{\Phi}^q(\Rd)\equiv\Banach_{\Phi}^p(\Rd)'$. In addition, $\overline{K(\vx,\cdot)}\in\Banach_{\Phi}^p(\Rd)$ for any $\vx\in\Rd$ because $\hat{\Phi}^{q-1}\in\Leb_1(\Rd)$ and $\left(\overline{K(\vx,\cdot)}\right)\hat{}=\hat{k}_{\vx}\in\Leb_q(\Rd;\mu)$ by $\Phi=\overline{\Phi(-\cdot)}$.

Finally, we verify the right-sided reproduction. Fix any $f\in\Banach_{\Phi}^p(\Rd)$ and $\vy\in\Rd$. We can verify that $\hat{f}\in\Leb_1(\Rd)$ as in the above proof.
Moreover, the continuity of $f$ and $\check{\hat{f}}$ allows us to recover $f$ pointwise from its Fourier transform via
\[
f(\vx)=\check{\hat{f}}(\vx)=(2\pi)^{-d/2}\int_{\Rd}\hat{f}(\vy)e^{i\vx^T\vy}\ud\vy.
\]
Thus, we have
\begin{align*}
&\langle f,K(\cdot,\vy) \rangle_{\Banach_{\Phi}^p(\Rd)}
=\langle \hat{f},\hat{k}_{\vy} \rangle_{\Leb_q(\Rd;\mu)}
=\int_{\Rd}\hat{f}(\vx)\overline{\hat{k}_{\vy}(\vx)}\ud\mu(\vx)\\
=&(2\pi)^{-d/2}\int_{\Rd}\frac{\hat{f}(\vx)\overline{\hat{\Phi}(\vx)e^{-i\vx^T\vy}}}{\hat{\Phi}(\vx)}\ud\vx
=(2\pi)^{-d/2}\int_{\Rd}\hat{f}(\vx)e^{i\vx^T\vy}\ud\vx=f(\vy).
\end{align*}
In the same way, we can also verify that $\Banach_{\Phi}^p(\Rd)$ satisfies the left-sided reproduction property, i.e.,
\[
\langle \overline{K(\vx,\cdot)},g \rangle_{\Banach_{\Phi}^p(\Rd)}
=\langle \hat{k}_{\vx},\hat{g} \rangle_{\Leb_q(\Rd;\mu)}
=\overline{\int_{\Rd}\hat{g}(\vy)\overline{\hat{k}_{\vx}(\vy)}\ud\mu(\vy)}
=\overline{g(\vx)},
\]
for all $g\in\Banach_{\Phi}^q(\Rd)\equiv\Banach_{\Phi}^p(\Rd)'$ and all $\vx\in\Rd$.
Therefore $\Banach_{\Phi}^p(\Rd)$ is an RKBS with the two-sided reproducing kernel $K$.

Since $\Banach_{\Phi}^p(\Rd)\equiv\Leb_q(\Rd;\mu)$ is reflexive and $K$ is even,
the dual space $\Banach_{\Phi}^p(\Rd)'\equiv\Banach_{\Phi}^q(\Rd)$ is also an RKBS with the two-sided reproducing kernel $K$.

Because $\Leb_q(\Rd;\mu)$ and $\Leb_p(\Rd;\mu)$ are uniformly convex and smooth by \cite[Theorem~5.2.11 and Example~5.4.8]{Megginson1998}. $\Banach_{\Phi}^p(\Rd)$ and $\Banach_{\Phi}^q(\Rd)$ are also uniformly convex and smooth.
\end{proof}

\begin{remark}\label{r:RKBS-PDF}
We can combine our result with \cite[Proposition~1.9.3]{Megginson1998} to conclude that the restriction of $\Banach_{\Phi}^p(\Rd)$ to the reals is also an RKBS with the two-sided reproducing kernel $K$ and its dual is isometrically equivalent to the restriction of $\Banach_{\Phi}^q(\Rd)$ to the reals.
It is well-known that the RKHS of a given reproducing kernel is unique. Theorem~\ref{t:RKBS-PDF}, however, shows that different RKBSs may have the same reproducing kernel. We will provide an example for this in Section~\ref{s:Matern}.
Moreover, the proof of Theorem~\ref{t:RKBS-PDF} provides that $\Banach_{\Phi}^p(\Rd)$ with $p\geq 2$ is still a right-sided RKBS without the additional condition $\hat{\Phi}^{q-1}\in\Leb_1(\Rd)$.

According to \cite[Theorem~10.10]{Wendland2005} any positive definite kernel can be used to construct an RKHS. We may extend the positive definite kernel into an RKBS.
\end{remark}

\begin{corollary}\label{c:RKBS-Lp-Lq}
Let $\Banach_{\Phi}^p(\Rd)$ with $p\geq2$ be defined in Theorem~\ref{t:RKBS-PDF}. Then $\Banach_{\Phi}^p(\Rd)\subseteq\Leb_p(\Rd)$.
\end{corollary}
%
\begin{proof}
We fix any $f\in\Banach_{\Phi}^p(\Rd)$. According to the proof of Theorem~\ref{t:RKBS-PDF}, we have $\hat{f}\in\Leb_q(\Rd)$ because
\[
\int_{\Rd}\abs{\hat{f}(\vx)}^q\ud\vx\leq(2\pi)^{qd/2}\left(\int_{\Rd}\frac{\abs{\hat{f}(\vx)}^q}{\hat{\Phi}(\vx)}\ud\vx\right)
\left(\sup_{\vx\in\Rd}\hat{\Phi}(\vx)\right)<\infty.
\]
The Hausdorff-Young inequality~\cite[Theorem~7.1.13]{Hormander2003I} provides that $f=\check{\hat{f}}\in\Leb_p(\Rd)$ because $1<q\leq 2$.
\end{proof}

\begin{remark}
The RKBS $\Banach_{\Phi}^p(\Rd)$ with $p\geq2$ can be precisely written as
\[
\begin{split}
\Banach_{\Phi}^p(\Rd):=&\left\{f\in\Leb_p(\Rd)\cap\Cont(\Rd): \ \text{the distributional Fourier transform $\hat{f}$ of $f$}\right.\\
&\left.\text{ is a measurable function defined on $\Rd$ such that }
\hat{f}\big/\hat{\Phi}^{1/q}\in\Leb_q(\Rd)\right\}.
\end{split}
\]
However, $\Banach_{\Phi}^p(\Rd)\not\subseteq\Leb_p(\Rd)$ with $1<p<2$ because the Hausdorff-Young inequality does not work for $q>2$.
\end{remark}

We fix any positive number $m>d/2$. According to \cite[Corollary~10.13]{Wendland2005}, if there are two positive constants $C_1,C_2$ such that
\[
C_1\left(1+\norm{\vx}_2^2\right)^{-m/2}\leq\hat{\Phi}(\vx)^{1/2}\leq C_2\left(1+\norm{\vx}_2^2\right)^{-m/2},
\quad \vx\in\Rd,
\]
then the RKHS $\Banach_{\Phi}^2(\Rd)\equiv\Hilbert_{\Phi}(\Rd)$ and the classical $\Leb_2$-based Sobolev space $W_2^m(\Rd)\equiv\Hilbert^m(\Rd)$ of order $m$ are isomorphic, i.e., $\Hilbert_{\Phi}(\Rd)\cong\Hilbert^m(\Rd)$.

Following the ideas of RKHSs, we can also find a relationship between RKBSs and Sobolev spaces. Let $f_m(\vx):=\left(1+\norm{\vx}_2^2\right)^{m/2}\hat{f}(\vx)$ with $p\geq2$. The theory of singular integrals then shows that $f$ belongs to the classical $\Leb_p$-based Sobolev space $W_p^m(\Rd)$ of order $m$ if any only if the function $f_m$ is the Fourier transform of some function in $\Leb_p(\Rd)$, and the $\Leb_p$-norm of the inverse Fourier transform $f_m$ is equivalent to the $W_p^m$-norm of $f$ (much more detail is mentioned in~\cite[Section~7.63]{AdamsFournier2003} and \cite[Section~7.9]{Hormander2003I}). Using the Hausdorff-Young inequality, we can get $\norm{f}_{W_p^m(\Rd)}\leq C\norm{\check{f}_m}_{\Leb_p(\Rd)}\leq C\norm{f_m}_{\Leb_q(\Rd)}$ for some positive constant $C$ independent of $f$. Following these statements, we can introduce the following corollary.

\begin{corollary}\label{c:RKBS-PDF-Sobolev}
Let the positive definite function $\Phi$ be as in Theorem~\ref{t:RKBS-PDF} and $W_p^m(\Rd)$ be the classical $\Leb_p$-based Sobolev space of order $m>pd/q-d/q$. Here $q$ is the conjugate exponent of $p\geq2$.
If there are two positive constants $C_1,C_2$ such that
\[
C_1\left(1+\norm{\vx}_2^2\right)^{-m/2}\leq\hat{\Phi}(\vx)^{1/q}\leq C_2\left(1+\norm{\vx}_2^2\right)^{-m/2},
\quad \vx\in\Rd,
\]
then $\Banach_{\Phi}^p(\Rd)$ is embedded into $W_p^m(\Rd)$, i.e.,
\[
\norm{f}_{W_p^m(\Rd)}\leq C\norm{f}_{\Banach_{\Phi}^p(\Rd)},\quad f\in\Banach_{\Phi}^p(\Rd)\subseteq W_p^m(\Rd),
\]
for some positive constant $C$ independent on $f$.
\end{corollary}

\begin{remark}\label{r:RKBS-PDF-Sobolev}
Here the lower bound for $m$ is induced by the condition that $\hat{\Phi}^{q/p}\in\Leb_1(\Rd)$.
According to Corollary~\ref{c:RKBS-PDF-Sobolev}, the dual space $W_q^{-m}(\Rd)$ of the Sobolev space $W_p^m(\Rd)$ is embedded into the dual space $\Banach_{\Phi}^p(\Rd)'$ of the RKBS $\Banach_{\Phi}^p(\Rd)$. It is well-known that the point evaluation functional $\delta_{\vx}$ belongs to $W_q^{-m}(\Rd)$ (see \cite[Section~3.25]{AdamsFournier2003}) which coincides with $\delta_{\vx}\in\Banach_{\Phi}^p(\Rd)'$.
\end{remark}

Since $K(\cdot,\vx_1),\ldots,K(\cdot,\vx_N)$ are linearly independent in $\Banach_{\Phi}^q(\Rd)\equiv\Banach_{\Phi}^p(\Rd)'$ for any pairwise distinct data points $X=\left\{\vx_1,\ldots,\vx_N\right\}\subseteq\Rd$, $\delta_{\vx_1},\ldots,\delta_{\vx_N}$ are linearly independent on $\Banach_{\Phi}^p(\Rd)$.
Combining Theorems~\ref{t:RKBS-opt-rep} and~\ref{t:RKBS-PDF}, we can solve the empirical SVM solution in $\Banach_{\Phi}^p(\Rd)$ with $p>1$.

\begin{theorem}\label{t:RKBS-PDF-opt-rep}
Let $\Banach_{\Phi}^p(\Rd)$ with $p>1$ be defined as in Theorem~\ref{t:RKBS-PDF} and the regularization function $R:[0,\infty)\to[0,\infty)$ be convex and strictly increasing. We choose the loss function $L:\Rd\times\CC\times\CC\to[0,\infty)$ such that $L(\vx,y,\cdot)$ is a convex map for any fixed $\vx\in\Rd$ and any fixed $y\in\CC$.
Given the data $D:=\left\{\left(\vx_1,y_1\right),\ldots,\left(\vx_N,y_N\right)\right\}$ with pairwise distinct data points $X=\left\{\vx_1,\ldots,\vx_N\right\}\subseteq\Rd$ and associated data values $Y=\left\{y_1,\ldots,y_N\right\}\subset\CC$,  the unique optimal solution (support vector machine solution) $s_{D,L,R}$ of
\begin{equation}\label{e:svm-PDF}
\min_{f\in\Banach_{\Phi}^p(\Rd)}\sum_{j=1}^NL\left(\vx_j,y_j,f(\vx_j)\right)+R\left(\norm{f}_{\Banach_{\Phi}^p(\Rd)}\right),
\end{equation}
has the explicit representation
\begin{equation}\label{e:svm-PDF-coef}
s_{D,L,R}(\vx)=(2\pi)^{-d/2}\int_{\Rd}\hat{\Phi}(\vy)^{p-1}\sum_{k=1}^Nc_ke^{i(\vx-\vx_k)^T\vy}\abs{\sum_{l=1}^Nc_le^{-i\vx_l^T\vy}}^{p-2}\ud\vy,
\quad \vx\in\Rd,
\end{equation}
for some coefficients $c_1,\ldots,c_N\in\CC$ and $i^2=-1$.
\end{theorem}
%
\begin{proof}

Using Theorems~\ref{t:RKBS-opt-rep} and~\ref{t:RKBS-PDF}, the dual element of the SVM solution $s_{D,L,R}$ of the SVM~\eqref{e:svm-PDF} is a linear combination of $K(\cdot,\vx_1),\ldots,K(\cdot,\vx_N)$, i.e.,
\[
s_{D,L,R}^{\ast}(\vx)=\sum_{k=1}^{N}b_kK(\vx,\vx_k)=\sum_{k=1}^Nb_k\Phi(\vx-\vx_k),\quad\vx\in\Rd,~\vb:=\left(b_1,\cdots,b_N\right)^T\in\CC^N.
\]
Suppose that $s_{D,L,R}$ is not trivial. According to the proof of Theorem~\ref{t:RKBS-PDF},
the identity element of $s_{D,L,R}^{\ast}\in\Banach_{\Phi}^q(\Rd)$ in $\Leb_p(\Rd;\mu)$ is the Fourier transform of $s_{D,L,R}^{\ast}$, i.e.,
\[
f_s(\vx):=\Fourier\left(s_{D,L,R}^{\ast}\right)(\vx)
=\sum_{k=1}^Nb_k\hat{\Phi}(\vx)e^{-i\vx^T\vx_k},\quad \vx\in\Rd.
\]
The dual element of $f_s\in\Leb_p(\Rd;\mu)$ in $\Leb_q(\Rd;\mu)$ has the form
\[
f_s^{\ast}(\vx)=\frac{f_s(\vx)\abs{f_s(\vx)}^{p-2}}{\norm{f_s}_{\Leb_p(\Rd;\mu)}^{p-2}},\quad \vx\in\Rd.
\]
Because the dual element of $s_{D,L,R}^{\ast}$ in $\Banach_{\Phi}^p(\Rd)$ is equal to the identity element of $f_s^{\ast}\in\Leb_q(\Rd;\mu)$ in $\Banach_{\Phi}^p(\Rd)$, which is the inverse Fourier transfer of $f_s^{\ast}$, we can determine that
\[
s_{D,L,R}(\vx)=\Fourier^{-1}\left(f_s^{\ast}\right)(\vx)
=(2\pi)^{-d/2}\int_{\Rd}\hat{\Phi}(\vy)^{p-1}\sum_{k=1}^Nc_ke^{i(\vx-\vx_k)^T\vy}\abs{\sum_{l=1}^Nc_le^{-i\vx_l^T\vy}}^{p-2}\ud\vy,
\]
and the coefficients are given by $c_k:=\norm{f_s}_{\Leb_p(\Rd;\mu)}^{\frac{2-p}{p-1}}b_k=\norm{s_{D,L,R}}_{\Banach_{\Phi}^p(\Rd)}^{q-2}b_k$ for all $k=1,\ldots,N$, where $q$ is the conjugate exponent of $p$.

\end{proof}

\begin{remark}\label{r:RKBS-PDF-opt-rep}
In particular, if $p$ is an even positive integer, then $s_{D,L,R}$ is also a linear combination of some kernel function translated to the data points $X$. For example, when $p=4$, then
\begin{align*}
s_{D,L,R}=\sum_{k_1,k_2,k_3=1}^{N,N,N}c_{k_1}\overline{c_{k_2}}c_{k_3}\Phi_3\left(\cdot-\vx_{k_1}+\vx_{k_2}-\vx_{k_3}\right)
=\sum_{k_1,k_2,k_3=1}^{N,N,N}c_{k_1}\overline{c_{k_2}}c_{k_3}\Ker_3\left(\cdot,\vx_{k_1},\vx_{k_2},\vx_{k_3}\right),
\end{align*}
where the kernel function $\Ker_3(\vx,\vy_1,\vy_2,\vy_3):=\Phi_3(\vx-\vy_1+\vy_2-\vy_3)$ and $\Phi_3$ is the inverse Fourier transform of $\hat{\Phi}^3$. Moreover,
\[
\begin{split}
&\norm{s_{D,L,R}}_{\Banach_{\Phi}^p(\Rd)}^{4/3}
=\norm{s_{D,L,R}}_{\Banach_{\Phi}^p(\Rd)}^{-2/3}[s_{D,L,R},s_{D,L,R}]_{\Banach_{\Phi}^p(\Rd)}
=\norm{s_{D,L,R}}_{\Banach_{\Phi}^p(\Rd)}^{-2/3}\langle s_{D,L,R},s_{D,L,R}^{\ast} \rangle_{\Banach_{\Phi}^p(\Rd)}\\
=&\sum_{j=1}^N\overline{c_j}\langle s_{D,L,R},K(\cdot,\vx_j) \rangle_{\Banach_{\Phi}^p(\Rd)}
=\sum_{j,k_1,k_2,k_3=1}^{N,N,N,N}\overline{c_j}c_{k_1}\overline{c_{k_2}}c_{k_3}\Ker_3\left(\vx_j,\vx_{k_1},\vx_{k_2},\vx_{k_3}\right).
\end{split}
\]
\end{remark}

We can observe that the coefficients of the SVM solution $s_{D,L,R}$ given in Theorem~\ref{t:RKBS-PDF-opt-rep} differ from the coefficients of its dual element $s_{D,L,R}^{\ast}$ only by a constant factor.
As in Corollary~\ref{c:RKBS-opt-coef-fixed-point}, the coefficients of $s_{D,L,R}$ can also be computed by the fixed point iteration method. For any fixed $\vc:=\left(c_1,\cdots,c_N\right)^T\in\CC^N$, we can define a unique function $s_{\vc}\in\Banach_{\Phi}^p(\Rd)$ as in Equation~\eqref{e:svm-PDF-coef}. Let
\[
\phi_j(\vc):=s_{\vc}(\vx_j)
=(2\pi)^{-d/2}\int_{\Rd}\hat{\Phi}(\vy)^{p-1}\sum_{k=1}^Nc_ke^{i(\vx_j-\vx_k)^T\vy}\abs{\sum_{l=1}^Nc_le^{-i\vx_l^T\vy}}^{p-2}\ud\vy,
\quad \vc\in\CC^N,
\]
for all $j=1,\ldots,N$, and $\vphi:=\left(\phi_1,\cdots,\phi_N\right)^T$. Thus we have
\[
\norm{s_{\vc}}_{\Banach_{\Phi}^p(\Rd)}^{q}
=\norm{s_{\vc}}_{\Banach_{\Phi}^p(\Rd)}^{q-2}\langle s_{\vc},s_{\vc}^{\ast} \rangle_{\Banach_{\Phi}^p(\Rd)}
=\sum_{j=1}^N\overline{c_j}\langle s_{\vc},K(\cdot,\vx_j) \rangle_{\Banach_{\Phi}^p(\Rd)}=\vc^{\ast}\vphi(\vc).
\]
Here $q$ is the conjugate exponent of $p$. Denote that
\[
\mT_{D,L,R}(\vc):=\sum_{j=1}^NL\left(\vx_j,y_j,\phi_j(\vc)\right)+R\left(\left(\vc^{\ast}\vphi(\vc)\right)^{1/q}\right)
=\sum_{j=1}^NL\left(\vx_j,y_j,s_{\vc}(\vx_j)\right)+R\left(\norm{s_{\vc}}_{\Banach_{\Phi}^p(\Rd)}\right).
\]
It is easy to check that the coefficients of $s_{D,L,R}$ are the minimizers of $\mT_{D,L,R}$ over $\CC^N$, i.e.,
\[
\vc_{opt}:=\underset{\vc\in\CC^N}{\text{argmin }}\mT_{D,L,R}(\vc)\text{ such that }s_{D,L,R}=s_{\vc_{opt}}.
\]
Suppose that $L(\vx,y,\cdot)\in\Cont^1(\CC)$ for all $\vx\in\Rd$ and all $y\in\CC$, $R\in\Cont^1([0,\infty))$ and $p\geq2$.
We can compute the gradient of $\mT_{D,L,R}$ by Wirtinger partial derivatives in the form
\[
\nabla\mT_{D,L,R}(\vc)^T=\vl'_{D}\left(\vphi(\vc)\right)^T\nabla\vphi(\vc)
+\frac{R'\left(\left(\vc^{\ast}\vphi(\vc)\right)^{1/q}\right)}{2q\left(\vc^{\ast}\vphi(\vc)\right)^{1/p}}\vc^{\ast}\nabla\vphi(\vc),
\]
where
$\vl'_D\left(\vphi\right):=\left(L'(\vx_1,y_1,\phi_1),\cdots,L'(\vx_N,y_N,\phi_N)\right)^T$ and the entries of the Jacobian (gradient) matrix $\nabla\vphi:=\left(\frac{\partial}{\partial c_k}\phi_j\right)_{j,k=1}^{N,N}$ by Wirtinger partial derivatives have the forms
\[
\frac{\partial}{\partial c_k}\phi_j(\vc)=
\frac{p}{2}(2\pi)^{-d/2}
\int_{\Rd}\hat{\Phi}(\vy)^{p-1}e^{i(\vx_j-\vx_k)^T\vy}
\abs{\sum_{l=1}^Nc_le^{-i\vx_l^T\vy}}^{p-2}
\ud\vy.
\]
Moreover, $\vc_{opt}$ is the stationary point of $\nabla\mT_{D,L,R}$ which indicates that $\vc_{opt}$ is a fixed point of the function
\begin{equation}\label{e:opt-coef-PDF-fixed-point}
F_{D,L,R}(\vc):=\vc+\nabla\mT_{D,L,R}(\vc),\quad \vc\in\CC^N.
\end{equation}
Therefore, we can introduce the following corollary.

\begin{corollary}\label{c:RKBS-PDF-opt-coef-fixed-point}

Suppose that the loss function $L(\vx,y,\cdot)\in\Cont^1(\CC)$ for all $\vx\in\Rd$ and all $y\in\CC$, the regularization function $R\in\Cont^1([0,\infty))$ and $p\geq2$.
Then the coefficient vector $\vc$ of the support vector machine solution $s_{D,L,R}$ given in Theorem~\ref{t:RKBS-PDF-opt-rep} is a fixed point of the function $F_{D,L,R}$ defined in Equation~\eqref{e:opt-coef-PDF-fixed-point}, i.e., $F_{D,L,R}(\vc)=\vc$.

\end{corollary}

\begin{remark}\label{r:RKBS-PDF-opt-coef-fixed-point}

The coefficients $\vc:=\left(c_1,\cdots,c_N\right)^T$ of the SVM solution $s_{D,L,R}$ in $\Banach_{\Phi}^p(\Rd)$ differ from the coefficients $\vb:=\left(b_1,\cdots,b_N\right)^T$ of its dual element $s_{D,L,R}^{\ast}$ in $\Banach_{\Phi}^q(\Rd)$ only by a constant factor.
Both coefficient vectors $\vb$ and $\vc$ are fixed points of the functions $F_{D,L,R}^{\ast}$ as in Equation~\eqref{e:opt-coef-fixed-point} and $F_{D,L,R}$ as in Equation~\eqref{e:opt-coef-PDF-fixed-point}, respectively. Roughly speaking, $F_{D,L,R}^{\ast}$ can be seen as a conjugate of $F_{D,L,R}$. Much more contents of these fixed point iteration algorithms for the binary classification problems will be deeply discussed in our next papers.

\end{remark}

We now use the techniques of \cite[Theorem~6]{BerlinetThomas2004} to set up a two-sided RKBS defined on a subset $\Domain$ of $\Rd$.

\begin{theorem}\label{t:RKBS-PDF-Omega}
Let the positive definite function $\Phi$ be as in Theorem~\ref{t:RKBS-PDF} and $\Domain\subseteq\Rd$.
Then the function space
\[
\Banach_{\Phi}^p(\Domain):=\left\{h:\text{ there exists a function }h\in\Banach_{\Phi}^p(\Rd)\text{ such that }f|_{\Domain}=h\right\},
\]
equipped with the norm
\[
\norm{h}_{\Banach_{\Phi}^p(\Domain)}:=\inf_{f\in\Banach_{\Phi}^p(\Rd)}\norm{f}_{\Banach_{\Phi}^p(\Rd)}\text{ s.t. }f|_{\Domain}=h,
\]
is a reproducing kernel Banach space with the two-sided reproducing kernel
\[
K|_{\Rd\times\Domain}(\vx,\vy):=\Phi(\vx-\vy),\quad \vx\in\Rd,~\vy\in\Domain,
\]
where $f|_{\Domain}$ stands for the restriction of $f$ to $\Domain$. Its dual space $\Banach_{\Phi}^p(\Domain)'$ is isometrically equivalent to a closed subspace of $\Banach_{\Phi}^q(\Rd)$ (the annihilator of $\Space_0$ in $\Banach_{\Phi}^q(\Rd)$)
\[
\Space_0^{\perp}=\left\{g\in\Banach_{\Phi}^q(\Rd)\equiv\Banach_{\Phi}^p(\Rd)':~\langle f,g \rangle_{\Banach_{\Phi}^p(\Rd)}=0,\text{ for all }f\in\Space_0\right\},
\]
where $q$ is the conjugate exponent of $p>1$ and
\[
\Space_0:=\left\{f\in\Banach_{\Phi}^p(\Rd):~f|_{\Domain}=0\right\}.
\]
Moreover, $\Banach_{\Phi}^p(\Domain)$ is uniformly convex and smooth.
\end{theorem}
%
\begin{proof}
Since convergence in a two-sided RKBS $\Banach_{\Phi}^p(\Rd)$ implies pointwise convergence, we can determine that $\Space_0$ is a closed subspace of $\Banach_{\Phi}^p(\Rd)$.
According to the construction of $\Banach_{\Phi}^p(\Domain)$, $\Banach_{\Phi}^p(\Domain)$ is isometrically equivalent to the quotient space $\Banach_{\Phi}^p(\Rd)\big/\Space_0$ (see \cite[Definition~1.7.1 and 1.7.3]{Megginson1998}). Thus $\Banach_{\Phi}^p(\Domain)$ is a Banach space by \cite[Theorem~1.7.9 and Corollary~1.11.19]{Megginson1998}.

Next we use the identification of $\left(\Banach_{\Phi}^p(\Rd)\big/\Space_0\right)'\equiv\Space_0^{\perp}$ to verify the two-sided reproduction (see \cite[Theorem~1.10.17]{Megginson1998}). Let $K$ be the reproducing kernel of $\Banach_{\Phi}^p(\Rd)$ given in Theorem~\ref{t:RKBS-PDF}.
We fix any $\vy\in\Domain$.
Since
\[
\langle f,K(\cdot,\vy) \rangle_{\Banach_{\Phi}^p(\Rd)}=f(\vy)=0,\quad \text{for all }f\in\Space_0,
\]
we have $K(\cdot,\vy)\in\Space_0^{\perp}\equiv\left(\Banach_{\Phi}^p(\Rd)\big/\Space_0\right)'\equiv\Banach_{\Phi}^p(\Domain)'$.
Combining this with the right-sided reproduction of $\Banach_{\Phi}^p(\Rd)$, we have
\[
\begin{split}
\langle h,K(\cdot,\vy) \rangle_{\Banach_{\Phi}^p(\Domain)}
=\langle Eh,K(\cdot,\vy) \rangle_{\Banach_{\Phi}^p(\Rd)}=(Eh)(\vy)=h(\vy),
\end{split}
\]
for all $h\in\Banach_{\Phi}^p(\Domain)$ and all $\vy\in\Domain$, where $E$ is the extension operator from $\Banach_{\Phi}^p(\Domain)$ into $\Banach_{\Phi}^p(\Rd)$ such that $Eh|_{\Domain}=h$ and $\norm{Eh}_{\Banach_{\Phi}^p(\Rd)}=\norm{h}_{\Banach_{\Phi}^p(\Domain)}$.
Since $\overline{K(\vx,\cdot)}|_{\Domain}\in\Banach_{\Phi}^p(\Domain)$ for all $\vx\in\Rd$, we can also obtain the left-sided reproduction of $\Banach_{\Phi}^p(\Domain)$, i.e.,
\[
\langle \overline{K(\vx,\cdot)}|_{\Domain},g \rangle_{\Banach_{\Phi}^p(\Domain)}
=\langle \overline{K(\vx,\cdot)},g \rangle_{\Banach_{\Phi}^p(\Rd)}=g(\vx),
\]
for all $g\in\Space_0^{\perp}\equiv\Banach_{\Phi}^p(\Domain)'$.
Therefore $\Banach_{\Phi}^p(\Domain)$ is an RKBS with the two-sided reproducing kernel $K|_{\Rd\times\Domain}$.

Since $\Banach_{\Phi}^p(\Rd)$ is uniformly convex, \cite[Theorem~5.2.24]{Megginson1998} provides that $\Banach_{\Phi}^p(\Domain)\equiv\Banach_{\Phi}^p(\Rd)\big/\Space_0$ is uniformly convex. We also know that $\Banach_{\Phi}^p(\Rd)'\equiv\Leb_q(\Rd;\mu)$ is uniformly convex and $\Space_0^{\perp}$ is a closed subspace of $\Banach_{\Phi}^q(\Rd)\equiv\Banach_{\Phi}^p(\Rd)'$ by \cite[Proposition~1.10.15]{Megginson1998}. Combining with \cite[Proposition~5.1.20 and~5.4.5]{Megginson1998}, we can also check that $\Banach_{\Phi}^p(\Domain)$ is smooth.
\end{proof}

\begin{remark}\label{r:RKBS-PDF-Omega}
When $p=2$, then we know that $\Banach_{\Phi}^2(\Domain)$ is a Hilbert space by Theorem~\ref{t:RKBS-PDF}. Thus the dual space and the space itself are isometrically isomorphic such that the reproducing kernel becomes $K|_{\Domain\times\Domain}$.
Since $\Banach_{\Phi}^2(\Rd)=\Space_0\oplus\Space_0^{\perp}$, we can determine that $\left\{g|_{\Domain}:~g\in\Space_0^{\perp}\right\}=\Banach_{\Phi}^2(\Domain)$ and $\norm{g}_{\Banach_{\Phi}^2(\Rd)}=\norm{g|_{\Domain}}_{\Banach_{\Phi}^2(\Domain)}$ for all $g\in\Space_0^{\perp}$ which implies that $\Banach_{\Phi}^2(\Domain)\equiv\Space_0^{\perp}\equiv\Banach_{\Phi}^2(\Domain)'$ and $\Banach_{\Phi}^2(\Domain)$ has the inner product
\[
(h_1,h_2)_{\Banach_{\Phi}^2(\Domain)}
=\langle h_1,h_2 \rangle_{\Banach_{\Phi}^2(\Domain)}
=\langle Eh_1,Eh_2 \rangle_{\Banach_{\Phi}^2(\Rd)}
=(Eh_1,Eh_2)_{\Banach_{\Phi}^2(\Rd)},
\]
for all $h_1,h_2\in\Banach_{\Phi}^2(\Domain)$. Therefore $\Banach_{\Phi}^2(\Domain)$ is an RKHS. Moreover, since $K(\cdot,\vy)\in\Space_0^{\perp}$ for any $\vy\in\Domain$, we have $E\left(K(\cdot,\vy)|_{\Domain}\right)=K(\cdot,\vy)$. This shows that $K|_{\Domain\times\Domain}$ is a reproducing kernel of $\Banach_{\Phi}^2(\Domain)$. This conclusion is the same as in \cite[Theorem~6]{BerlinetThomas2004}.

If the RKBS is even a Hilbert space, then we can choose an equivalent function space of its dual as itself such that its reproducing kernel has symmetric domains. The difficulty to find an equivalent function space of the dual of RKBS, which is defined on the same domain of the RKBS, causes the domains of its reproducing kernel to be nonsymmetric. Theorems~\ref{t:RKBS-PDF} and~\ref{t:RKBS-PDF-Omega} provide us with examples of symmetric and nonsymmetric reproducing kernels of RKBSs, respectively.
\end{remark}

Suppose that the positive definite function $\Phi$ given in Theorem~\ref{t:RKBS-PDF} has a compact support $\Domain_{\Phi}$. Because of the positive definite properties of $\Phi$, its support $\text{supp}(\Phi)=\Domain_{\Phi}$ with the origin is symmetric and bounded. Let $\Domain_1$ and $\Domain_2$ be two subsets of $\Rd$ such that the complement $\Domain_1^c$ includes $\Domain_2^c+\Domain_{\Phi}$. We fix any $\gamma\in\Schwartz$ so that its support $\text{supp}(\gamma)\subseteq\Domain_2^c$. Since the convolution function $\gamma\ast\Phi\in\Banach_{\Phi}^p(\Rd)$ and its support $\text{supp}(\gamma\ast\Phi)\subseteq\text{supp}(\gamma)+\text{supp}(\Phi)\subseteq\Domain_2^c+\Domain_{\Phi}\subseteq\Domain_1^c$, we can determine that $\gamma\ast\Phi\in\Space_0$ with $\Domain:=\Domain_1$. For any $g\in\Space_0^{\perp}$, we have
\[
\int_{\Rd}\gamma(\vx)\overline{g(\vx)}\ud\vx
=\int_{\Rd}\hat{\gamma}(\vx)\overline{\hat{g}(\vx)}\ud\vx
=\int_{\Rd}\frac{\widehat{\gamma\ast\Phi}(\vx)\overline{\hat{g}(\vx)}}{\hat{\Phi}(\vx)}\ud\vx
=\langle \gamma\ast\Phi,g\rangle_{\Banach_{\Phi}^p(\Rd)}
=0
\]
which indicates that $g|_{\Domain_2^c}=0$. According to this result we can deduce that $g=0$ if and only if $g\in\Space_0^{\perp}$ and $g|_{\Domain_2}=0$.
This means that the restriction map of $\Space_0^{\perp}$ to $\Domain_2$ is one-to-one.
Thus the normed space
\[
\Banach(\Domain_2):=\left\{\phi:\Domain_2\to\CC:~\phi=g|_{\Domain_2}\text{ for some }g\in\Space_0^{\perp}\right\}
\]
equipped with the norm
$\norm{\phi}_{\Banach(\Domain_2)}:=\norm{g}_{\Banach_{\Phi}^q(\Rd)}$ is well-defined and it is obvious that $\Banach(\Domain_2)\equiv\Space_0^{\perp}$.
Under these additional conditions, the dual space of $\Banach_{\Phi}^p(\Domain_1)$ defined in Theorem~\ref{t:RKBS-PDF-Omega} can be even isometrically equivalent to a space composed of functions defined on $\Domain_2$, i.e., $\Banach_{\Phi}^p(\Domain_1)'\equiv\Space_0^{\perp}\equiv\Banach(\Domain_2)$. In this case $\Banach_{\Phi}^p(\Domain_1)$ is also an RKBS with the two-sided reproducing kernel $K|_{\Domain_2\times\Domain_1}$.

\begin{corollary}\label{c:RKBS-PDF-Omega}
Suppose that the positive definite function $\Phi$ given in Theorem~\ref{t:RKBS-PDF} has a compact support $\Domain_{\Phi}$ in $\Rd$.
Let $\Domain_1$ and $\Domain_2$ be two subsets of $\Rd$ such that the complement $\Domain_1^c$ includes $\Domain_2^c+\Domain_{\Phi}$.
Then $\Banach_{\Phi}^p(\Domain_1)$ with $p>1$ defined in Theorem~\ref{t:RKBS-PDF-Omega}
is a reproducing kernel Banach space with the two-sided reproducing kernel
\[
K|_{\Domain_2\times\Domain_1}(\vx,\vy):=\Phi(\vx-\vy),\quad \vx\in\Domain_2,~\vy\in\Domain_1.
\]
\end{corollary}

If the subset $\Domain$ is a \emph{regular} domain, then the definition of weak derivatives (see \cite[Section~1.62]{AdamsFournier2003}) provides that $f|_{\Domain}\in W_p^m(\Domain)$ and $\norm{f|_{\Domain}}_{W_p^m(\Domain)}\leq \norm{f}_{W_p^m(\Rd)}$ for all $f\in W_p^m(\Rd)$, where $W_p^m(\Domain)$ is the $\Leb_p$-based Sobolev space of order $m$.
Now we use the embeddings of $\Banach_{\Phi}^p(\Rd)$ given in Corollary~\ref{c:RKBS-PDF-Sobolev} to derive the embeddings of $\Banach_{\Phi}^p(\Domain)$. We fix any $h\in\Banach_{\Phi}^p(\Domain)$. According to Corollary~\ref{c:RKBS-PDF-Sobolev}, we have
\[
\norm{h}_{W_p^m(\Domain)}\leq\norm{Eh}_{W_p^m(\Rd)}\leq C\norm{Eh}_{\Banach_{\Phi}^p(\Rd)}=C\norm{h}_{\Banach_{\Phi}^p(\Domain)},\quad h\in\Banach_{\Phi}^p(\Domain)\subseteq W_p^m(\Domain),
\]
for some positive constant $C$ independent on $h$.

\begin{corollary}\label{c:RKBS-PDF-Sobolev-Omega}
Let $\Phi$ be a positive definite function and $m>pd/q-d/q$ be as in Corollary~\ref{c:RKBS-PDF-Sobolev}. Here $q$ is the conjugate exponent of $p\geq2$. Suppose that $\Domain\subseteq\Rd$ is regular.
Then $\Banach_{\Phi}^p(\Domain)$ defined in Theorem~\ref{t:RKBS-PDF-Omega} is embedded into the $\Leb_p$-based Sobolev space of order $m$, $W_p^m(\Domain)$, i.e.,
\[
\norm{h}_{W_p^m(\Domain)}\leq C\norm{h}_{\Banach_{\Phi}^p(\Domain)},\quad h\in\Banach_{\Phi}^p(\Domain)\subseteq W_p^m(\Domain),
\]
for some positive constant $C$ independent on $h$.
\end{corollary}

\section{Examples for Mat\'ern Functions}\label{s:Matern}

\cite[Example~5.7]{FasshauerYe2011Dist} and \cite[Example~4.4]{YePhD2012} show that \emph{Mat\'ern functions} (Sobolev splines) with shape parameter $\theta>0$ and degree $n>d/2$
\[
G_{\theta,n}(\vx):=\frac{2^{1-n-d/2}}{\pi^{d/2}\Gamma(n)\theta^{2n-d}}(\theta\norm{\vx}_2)^{n-d/2}K_{d/2-n}(\theta\norm{\vx}_2),
\quad \vx\in\Rd,
\]
are positive definite functions on $\Rd$, where $t\mapsto K_{\nu}(t)$ is the modified Bessel function of the second kind of order $\nu$ and $t\mapsto\Gamma(t)$ is the Gamma function. Moreover, $G_{\theta,n}$ is a full-space Green function of the differential operator $L_{\theta,n}:=\left(\theta^2I-\Delta\right)^n$, i.e., $L_{\theta,n}G_{\theta,n}=\delta_{\v0}$. The Fourier transform of $G_{\theta,n}$ has the form
\[
\hat{G}_{\theta,n}(\vx)=\left(\theta^2+\norm{\vx}_2^2\right)^{-n},\quad \vx\in\Rd.
\]

Let $1<q\leq2\leq p<\infty$ with $p^{-1}+q^{-1}=1$ such that $nq/p>d/2$ and $m:=2n/q$. Since $\hat{G}_{\theta,n}^{\min\{p,q\}-1}\in\Leb_1(\Rd)$, Theorem~\ref{t:RKBS-PDF} provides that $\Banach_{G_{\theta,n}}^p(\Rd)$ is an RKBS on $\Rd$ with the two-sided reproducing kernel $K_{\theta,n}(\vx,\vy)=G_{\theta,n}(\vx-\vy)$. We can also check that there are two positive constants $C_1,C_2$ such that
\[
C_1\left(1+\norm{\vx}_2^2\right)^{-m/2}\leq\hat{G}_{\theta,n}(\vx)^{1/q}\leq C_2\left(1+\norm{\vx}_2^2\right)^{-m/2}.
\quad \vx\in\Rd.
\]
According to Corollary~\ref{c:RKBS-PDF-Sobolev} and~\ref{c:RKBS-PDF-Sobolev-Omega}, the RKBS $\Banach_{G_{\theta,n}}^p(\Rd)$ is embedded into $W_p^m(\Rd)$ and the RKBS $\Banach_{G_{\theta,n}}^p(\Domain)$ is embedded into $W_p^m(\Domain)$ for any regular domain $\Domain$ of $\Rd$.

In particular, when $p:=4$, then $\hat{G}_{\theta,n}^3=\hat{G}_{\theta,3n}$. According to the discussion of Theorem~\ref{t:RKBS-PDF-opt-rep} and Remark~\ref{r:RKBS-PDF-opt-rep}, the optimal solution of the SVM
\[
\min_{f\in\Banach_{G_{\theta,n}}^4(\Rd)}\sum_{j=1}^NL(\vx_j,y_j,f(\vx_j))+R\left(\norm{f}_{\Banach_{G_{\theta,n}}^4(\Rd)}\right),
\]
has the explicit representation
\begin{align*}
s_{D,L,R}(\vx) & =\sum_{k_1,k_2,k_3=1}^{N,N,N}c_{k_1}\overline{c_{k_2}}c_{k_3}G_{\theta,3n}\left(\vx-\vx_{k_1}+\vx_{k_2}-\vx_{k_3}\right) \\
& =\sum_{k_1,k_2,k_3=1}^{N,N,N}c_{k_1}\overline{c_{k_2}}c_{k_3}\Ker_{\theta,3n}\left(\vx,\vx_{k_1},\vx_{k_2},\vx_{k_3}\right),
\quad \vx\in\Rd,
\end{align*}
and its coefficients $\vc=\left(c_1,\cdots,c_N\right)^T$ are obtained by solving the following minimization problem
\[
\begin{split}
\min_{\vc\in\CC^N}\sum_{j=1}^NL\left(\vx_j,y_j,\sum_{k_1,k_2,k_3=1}^{N,N,N}c_{k_1}\overline{c_{k_2}}c_{k_3}
\Ker_{\theta,3n}\left(\vx_j,\vx_{k_1},\vx_{k_2},\vx_{k_3}\right)\right) \\
+R\left(\sum_{j,k_1,k_2,k_3=1}^{N,N,N,N}\overline{c_j}c_{k_1}\overline{c_{k_2}}c_{k_3}
\Ker_{\theta,3n}\left(\vx_j,\vx_{k_1},\vx_{k_2},\vx_{k_3}\right)\right)^{3/4},
\end{split}
\]
where $\Ker_{\theta,3n}(\vx,\vy_1,\vy_2,\vy_3):=G_{\theta,3n}(\vx-\vy_1+\vy_2-\vy_3)$, and the loss function $L$ and the regularization function $R$ are the same as in Theorem~\ref{t:RKBS-PDF-opt-rep}. More generally, when $p$ is even, then the SVM solution $s_{D,L,R}$ in $\Banach_{G_{\theta,n}}^{p}(\Rd)$ is a linear combination of the product groups of the reproducing kernel bases, i.e.,
\[
s_{D,L,R}(\vx)=\sum_{\vk\in\mathscr{G}_{p-1}^N}\prod_{j=1}^{p/2}c_{k_{2j-1}}\prod_{l=1}^{p/2-1}\overline{c_{k_{2l}}}
\Ker_{\theta,(p-1)n}\left(\vx,\vx_{k_1},\cdots,\vx_{k_{p-1}}\right),\quad \vx\in\Rd,
\]
where $\Ker_{\theta,(p-1)n}\left(\vx,\vy_1,\cdots,\vy_{p-1}\right):=G_{\theta,(p-1)n}\left(\vx-\vy_1+\vy_2+\cdots+(-1)^{p-1}\vy_{p-1}\right)$ and $\mathscr{G}_{p-1}^N:=\left\{\vk:=(k_1,\cdots,k_{p-1})^T\in\NN^{p-1}:~1\leq k_j\leq N,~j=1,\ldots,p-1\right\}$.

According to some numerical experiments comparing $\Banach_{G_{\theta,n}}^{2}(\RR^2)$ and $\Banach_{G_{\theta,n}}^{4}(\RR^2)$, we find that the accuracy of the SVM solutions in $\Banach_{G_{\theta,n}}^{4}(\RR^2)$ is better than in $\Banach_{G_{\theta,n}}^{2}(\RR^2)$ for the same training data and testing data. The reason for this is that we use three data points to set up each reproducing kernel base for $p=4$ but the reproducing kernel base for $p=2$ only owns two data points. This means that the reproducing kernel base for $p=4$ contains much more information than for $p=2$. Many other numerical tests will appear in a future paper.

The Mat\'ern functions have been applied in the field of statistical learning (see~\cite{Stein1999}). This new discovery about Mat\'ern functions might help create new numerical tools for SVMs in RKBS.

\section*{Acknowledgments}

The third author would like to express his gratitude to Prof. Xu, Yuesheng (Syracuse) who provided valuable suggestions that allowed us to make significant improvements to this paper.












Gregory E. Fasshauer, Fred J. Hickernell\\
Department of Applied Mathematics, Illinois Institute of
Technology,
Chicago, Illinois 60616 \\
E-mail address: fasshauer@iit.edu, hickernell@iit.edu\\

Qi Ye\\
Mathematics Department, Syracuse University,
Syracuse, NY 13244 \\
E-mail address: qiye@syr.edu\\


\end{document}